\documentclass{article}

\usepackage{Packages} 
\usepackage{Maths}

\title{Symmetric Linear Bandits with Hidden Symmetry}

\author{%
  Nam Phuong Tran \\
  Department of Computer Science\\
  University of Warwick\\
  Coventry, United Kingdom\\
  \texttt{nam.p.tran@warwick.ac.uk} \\
   \And
   The Anh Ta \\
   CSIRO's Data61\\
   Marsfield, NSW, Australia \\
    \texttt{ theanh.ta@csiro.au}\\
   \And
   Debmalya Mandal\\
   Department of Computer Science\\
   University of Warwick\\
   Coventry, United Kingdom\\
   \texttt{debmalya.mandal@warwick.ac.uk}\\
   \And
   Long Tran-Thanh\\
   Department of Computer Science\\
   University of Warwick\\
   Coventry, United Kingdom\\
   \texttt{long.tran-thanh@warwick.ac.uk}
}

\begin{document}

\maketitle
\addtocontents{toc}{\protect\setcounter{tocdepth}{0}}

\begin{abstract}
    High-dimensional linear bandits with low-dimensional structure have received considerable attention in recent studies due to their practical significance.
    The most common structure in the literature is sparsity. However, it may not be available in practice. 
    Symmetry, where the reward is invariant under certain groups of transformations on the set of arms, is another important inductive bias in the high-dimensional case that covers many standard structures, including sparsity.
    In this work, we study high-dimensional symmetric linear bandits where the symmetry is hidden from the learner, and the correct symmetry needs to be learned in an online setting.
    We examine the structure of a collection of hidden symmetry and provide a method based on model selection within the collection of low-dimensional subspaces.
    Our algorithm achieves a regret bound of $ O(\DimensionTrueFSS^{2/3} T^{2/3} \log(d))$, where $d$ is the ambient dimension which is potentially very large, and $\DimensionTrueFSS$ is the dimension of the true low-dimensional subspace such that $\DimensionTrueFSS \ll d$. 
    With an extra assumption on well-separated models, we can further improve the regret to $ O(\DimensionTrueFSS \sqrt{T\log(d)} )$. 
\end{abstract}

\section{Introduction}

Stochastic bandit is a sequential decision-making problem
in which a player, who aims to maximize her reward, selects an action at each step and receives a stochastic reward, drawn from an initially unknown distribution of the selected arm, in response.
Linear stochastic bandit (LSB)~\citep{Abbasi2011_OFUL} is an important variant in which the expected value of the reward is a linear function of the action. It is one of the most studied bandit variants and has many practical applications \citep{Lattimore2020_BanditBook}.
%

Actions in LSB are specified as feature vectors in $\mathbb{R}^d$ for very large feature dimension $d$, with performance i.e. the resulting regret scaling with $d$. Many works have addressed this curse of dimensionality by leveraging different low-dimensional structures as inductive biases for the learner. 
For example, sparsity, which assumes that the reward is a sparse linear function, has been used extensively in LSB to design bandit algorithms with better performance~\citep{Yasin_OnlineToConfidentSetSparseBandit,Oh2020_AgnositcSparseLinearBandit,Hao2020_SparseLinBandit_PoorRegime}. 
However, when the reward function lacks the structure for sparsity (which may occur in many real-world situations), a question arises: 
Are there different structures in the features of LSB that we can exploit to overcome the curse of dimensionality and design  bandit algorithms with better performance?

In this paper, we study the inductive bias induced by symmetry structures in LSB, which is a more general model inductive bias than sparsity, and can facilitate efficient and effective learning~\citep{bronstein2021geometric}.
Symmetry describes how, under certain transformations of the input of the problem, the outcome should either remain unchanged (\emph{invariance}) or shift predictably (\emph{equivariance}).
In supervised learning, it has been empirically observed \citep{Finzi2020_SymmetrySLPratice, Cesa2022_SymmetrySLPratice} and theoretically proven \citep{Elesedy2021_SymmetricAntiSymmetricDecomposition, Behboodi2022_SymmetrySupversiedLearning} that explicitly integrating symmetry into models leads to improved generalization error. 
However, in the literature on sequential decision-making, unlike sparsity, symmetry is rarely considered to date.
This leads us to the following research question: Can one leverage \emph{symmetry in sequential decision-making} tasks to enable effective exploration, and eventually break the curse of dimensionality?

In the machine learning literature, especially in supervised learning, most studies on symmetry assume prior knowledge of symmetry structures of the tasks under consideration \citep{Li2019_DataAugmentation,Elesedy2021_SymmetricAntiSymmetricDecomposition}. 
However, in numerous practical scenarios, the learner can only access to partial knowledge of the symmetry, necessitating the incorporation of symmetry learning mechanisms into the algorithms to achieve better performance. 
Examples of hidden symmetry can be found in multi-agent learning with cooperative behavior.
As a motivating example, consider a company undertaking a large project that consists of several subtasks. 
The company must hire subcontractors with the goal of maximizing project quality while staying within budget constraints. 
Symmetry may arise in this situation when coalitions form among subcontractors, where members of a coalition work together to complete their allocated tasks using shared resources. 
In particular, the allocation of tasks within a coalition can be swapped without affecting overall team performance, inducing symmetry (i.e., performance remains invariant under permutation) in the task assignments.
Coalitions among subcontractors often arise since sharing labor and resources reduces operational costs, making their work more efficient and cost-effective. 
However, these coalitions are typically \textit{hidden} from the hiring company. 
One reason is that if the hiring company were aware of these collaborations, they could use this information to negotiate lower prices, knowing that the subcontractors are benefiting from shared resources. 
Another reason is that coalitions may raise concerns about collusion. 
In particular, in a competitive market, such as when subcontractors are hired through a bidding platform, coalition members can collaborate to manipulate the bidding process, which is considered unfair and could undermine the integrity of the bidding.
For more practical examples of hidden symmetry in multi-agent reinforcement learning \cite{Mondal2022_MeanFieldGamesClassHomogeneous} and robotics \cite{Abreu2023_UnknownSymmetryRobotic}, we refer the reader to Appendix \ref{appendix: Practical Examples of Partitions with Subexponential-Size}.
Motivated by these examples, we believe that hidden symmetry is much more relevant in the context of sequential decision-making because the environment and its symmetry structure may not be readily available to the learner, as opposed to supervised learning and offline settings where data are provided during the training phase.
As the learner has the power to freely collect data, it is expected that they will learn the hidden symmetry structures as they explore the environment.

Against this background, we ask the question of whether learner can leverage symmetry to enable effective exploration, and break the curse of dimensionality \emph{without a  prior knowledge} of the symmetry structure? Moreover, in the presence of symmetry, when can we design learning algorithms with optimal regret bounds?
Towards answering this question, we investigate the setting of symmetric  linear stochastic bandit in $\mathbb{R}^d$, where $d$ is potentially very large, and the expected reward function is invariant with respect to the actions of a \textit{hidden} group $\mcal G$ of coordinate permutations.
Our contributions are summarised as follows:
\begin{enumerate}
   \item We first give an impossibility result that \emph{no algorithm can get any benefit by solely knowing that $\mcal G$ is a subgroup of permutation matrices}.
    We achieve this by formally establishing a relation between the class of subgroup to the partition over the set $\{1,...,d\}$.
    A direct implication of this impossibility result is that it is necessary to have further information about the structure of the hidden subgroup in order to achieve improved regret bounds.
    \item Given this, we establish a cardinality condition on the class of symmetric linear bandits with hidden $\mathcal{G}$, in which the learner can learn the true symmetry structure and overcome the curse of dimensionality. 
    Notably, this class includes a sparsity as a special case, and therefore inherits all the computational and statistical complexities of sparsity.
    Apart from sparsity, this class includes many other practically relevant classes, such as partitions that respect underlying hierarchies, non-crossing partitions, and non-nesting partitions (see Subsection \ref{subsubsec: assumptions} and Appendix \ref{appendix: Practical Examples of Partitions with Subexponential-Size}).
    \item We cast our problem of learning with hidden subgroup $\mcal G$ into model selection with collection of low-dimensional subspaces \cite{Kassraie2023_FeatureSelection_LinBan_LogM_Sparsity, Moradipari2021_LinearBanditRepresentationLearning}.
To address the polynomial scaling of regret bounds with respect to the number of models and arms in previous works, 
we depart from model aggregation, which is typically used in LSB model selection, and introduce a new framework inspired by Gaussian model selection \cite{Lucien_GaussianModelSelection}
\footnote{We note that \say{Gaussian model selection} is a technique in statistics, similar to model aggregation (see \cite{Giraud2021_HighDimStatBook}'s chapter 2 and 4), which should not be confused with \say{model selection} in the bandit literature.}
and compressed sensing~\cite{Blumensath2009_RIP_Concentration_UoS}. 
Based on this framework, we introduce a new algorithm, called EMC (for Explore Models then Commit).
Under the assumption that the set of arm is exploratory, we prove that the regret bound of the EMC algorithm is $O(d_0^{2/3}T^{2/3}\log(d))$, and $O(d_0\sqrt{T}\log(d))$ with an additional assumption on well-separated partitions, where $d_0 \ll d$ is the dimension of the low-dimensional subspace associated with group $\mcal G$.
\end{enumerate}
To the best of our knowledge, our work is the first in the linear stochastic bandits literature that leverages symmetry in designing provably efficient algorithms.
To save space, all proofs in this paper are deferred to the Appendix.

\color{black}


\subsection{Related Work}
We now briefly outline related work and compare them with our results. We refer the reader to Appendix \ref{appendix: Extended Related Work} for a more in-depth literature review.

\textbf{Sparse linear bandits.} As we will explain in Section \ref{sec: Regret Analysis of EMC algorithm}, sparsity is equivalent to a subset symmetry structures, and thus, can be seen as a special case of our setting. As such, we first review the literature of sparsity.
Sparse linear bandits were first investigated in \cite{Yasin_OnlineToConfidentSetSparseBandit}, where the authors achieve a regret of $\tilde{O}(\sqrt{dsT})$, with $\tilde{O}$ disregarding the logarithmic factor, and $s$ representing the sparsity level, and $T$ is the time horizon.
This matches the regret lower bound for sparse bandits, which is $\Omega(\sqrt{dsT})$ \citep{Lattimore2020_BanditBook}. 
 More recently, the contextual version of linear bandits has gained popularity, where additional assumptions are made regarding the context distribution and set of arms \citep{Kim2019_SparseLinBanditcompatibility, Oh2020_AgnositcSparseLinearBandit,Lattimore2015_SparseLinBanditCube, Carpentier2012_SparseLinBanditonSphere, Hao2020_SparseLinBandit_PoorRegime} to avoid polynomial dependence on $d$.
Notably, with the assumption on exploratory set of arms, \cite{Hao2020_SparseLinBandit_PoorRegime} propose an Explore then Commit style strategy that achieves $\Tilde O(s^{\frac{2}{3}}T^{\frac{2}{3}})$, nearly matching the regret lower bound $\Omega(s^{\frac{2}{3}}T^{\frac{2}{3}})$ \cite{Kyoungseok2022_PopArtSparseBandit} in the data-poor regime.
As sparsity is equivalent to a subclass of hidden symmetry, all the lower bounds for sparse problems apply to our setting of learning with hidden symmetry.

\textbf{Model selection.} Our problem is also closely related to the problem of model selection in linear bandits, as the learner can collect potential candidates for the hidden symmetry model.
Particularly, in model selection, there is a collection of $M$ features, and different linear bandits running with each of these features serve as base algorithms.
By exploiting the fact that the data can be shared across all the base algorithms, the dependence of regret in terms of the number of features can be reduced to $\log(M)$.
In particular, \cite{Kassraie2023_FeatureSelection_LinBan_LogM_Sparsity} propose a method that concatenates all $M$ features of dimension $d$ into one feature of dimension $Md$, and uses the Lasso estimation as a aggregation of models. 
Their algorithm achieves a regret bound of $O(T^{\frac{3}{4}} \sqrt{\log(M)})$ under the assumption that the Euclidean norm of the concatenated feature is bounded by a constant.
However, in our case, the Euclidean norm of the concatenated feature vector can be as large as $\sqrt{M}$, which leads to a $\sqrt{M}$ multiplicative factor in the regret bound.
Besides, \cite{Moradipari2021_LinearBanditRepresentationLearning} uses the online aggregation oracle approach, and is able to obtain regret of $O(\sqrt{KdT\log(M)})$, where $K$ is the number of arms.
In contrast, \emph{we use algorithmic mechanisms that are different from aggregation of models}. 
In particular, we explicitly exploit the structure of the model class as a collection of subspaces and invoke results from Gaussian model selection \citep{Giraud2021_HighDimStatBook, Lucien_GaussianModelSelection} and dimension reduction on the union of subspaces \citep{Blumensath2009_RIP_Concentration_UoS}.
With this technique, we are able to achieve $O(T^{\frac{2}{3}}\log(M))$, which is rate-optimal in the data-poor regime, has logarithmic dependence on $M$ without strong assumptions on the norm of concatenated features, and is independent of the number of arms $K$.
We refer the reader to Section \ref{sec: Regret Analysis of EMC algorithm} for a more detailed explanation.

\textbf{Symmetry in online learning.}  The notion of symmetry in Markov decision process dates back to works such as \cite{Givan2003_MDPBisimulation, Ravindran2004_MDPHomomorphism}. 
Generally, the reward function and probability transition are preserved under an action of a group on the state-action space. 
Exploiting known symmetry has been shown to help achieve better performance empirically \citep{Elise20_MDPHHomomorphism, Elise2020_MDPHomomorphismPlaning} or tighter regret bounds theoretically \citep{Tran2022_ILB}.
However, all these works requires knowledge of symmetry group, while our setting consider hidden symmetry group which may be considerably harder.
Hidden symmetry on the context or state space has been studied by few authors, with the term context-lumpable bandits \citep{Lee2023_Context_lumpable_bandit}, meaning that the set of contexts can be partitioned into classes of similar contexts.
It is important to note that the symmetry group acts differently on the context space and the action space.
As we shall explain in detail in Section \ref{sec: Partition and FixedSS, imposibility}, while one can achieve a reduction in terms of regret in the case of hidden symmetry acting on context spaces \citep{Lee2023_Context_lumpable_bandit}, this is not the case when the symmetry group acts on the action space.
The work closest to ours is \cite{Pesquerel2021_MABsimilarArmsConstraintPartition}, where the authors consider the setting of a $K$-armed bandit, where the set of arms can be partitioned into groups with similar mean rewards, such that each group has at least $q > 2$ arms. With the constrained partition, the instance-dependent regret bounds are shown asymptotically to be of order $O\left(\tfrac{K}{q} \log T\right)$.
Comparing to \cite{Pesquerel2021_MABsimilarArmsConstraintPartition}, we study the setting of stochastic linear bandits with similar arms, in which the (hidden) symmetry and linearity structure may intertwine, making the problem more sophisticated. We also impose different constraints on the way one partitions the set of arms, which is more natural in the setting of linear bandits with infinite arms.

\section{Problem Setting}
For any $k \in \mbb N^+$, denote $[k] = \{1,\dots,k \}$. 
For $\mcal X \subseteq \mbb R^d$, let $\Delta(\mcal X)$ denote the set of all probability measures supported on $\mcal X$.
Given a set $S\subset \mbb R^k$, for some $k>1$, denote $\Pi_S(x)$ as the Euclidean projection of $x \in \mbb R^k$ on $S$, and $\mrm{conv}(S)$ as the convex hull of $S$. 

We denote by $T$ the number of rounds, which  is assumed to be known in advance. 
Each round $t \in [T]$, the agent chooses an arm $x_t \in \mcal X \subset \mbb{R}^{d}$, and nature returns a stochastic reward
 $   y_t  = \AngleBr{x_t, \theta_\star} + \eta_t, \, $
where $\eta_t$ is an i.i.d. $\sigma$-Gaussian random variable.
Now, denote $f(x_t) = \Expectation[y_t \mid x_t]$.
A bandit strategy is a decision rule for choosing an arm $x_t$ in round $t \in [T]$, given past observations up to round $t-1$.
Formally, a bandit strategy is a mapping $\mcal A : (\mcal X \times \mbb R)^{T} \rightarrow \Delta(\mcal X)$.

Let $x_\star = \argmax_{x \in \mcal X} f(x)$, and let $\mbf R_T = \mbb E\SquareBr{\sum_{t=1}^T \AngleBr{ x_\star - x_t, \theta_\star} }$ denote the expected cumulative regret.
In this paper, we investigate the question whether one can get any reduction in term of regret, if the reward function is invariant under the action of a hidden group of transformations on the set of arms. 
We define the notion of group of symmetry as follows:
\paragraph{Group and group action.} Given $d \in \mathbb{N}^+$, let $\mcal S_d$ denote the symmetry group of $[d]$, that is, $\mcal S_d:=\{h: [d] \rightarrow [d]\mid \text{$h$ is bijective}\}$ the collection of all bijective mappings from $[d]$ to itself.
We also define the group action $\hat\phi$ of $\mcal S_d$ on the vector space $\mbb R^d$ as 
\begin{equation}
    \begin{aligned}
        \phi:\; &\mcal S_d \times \mbb R^d &&\rightarrow \mbb R^d \\
        &\RoundBr{g, (x_i)_{i\in [d]}} &&\mapsto (x_{g(i)})_{i\in [d]}
    \end{aligned}
\end{equation}
In other words, a group element $g$ acts on an arm $x \in \mbb{R}^d$ by permuting the coordinates of $x$. 
In the setting of linear bandit, the permutation group action also acts on the set of parameters via coordinate permutation.
For brevity, we simply denote $g \cdot \theta$ and $g \cdot x$ as $\phi(g, \theta)$ and $\phi(g, x)$, respectively. 
Denote by $A_g$ the permutation matrix corresponding to $g$.
We write $\mcal G \leq \mcal S_d$ to denote that $\mcal G$ is a subgroup of $\mcal S_d$. 
Given any point $\theta \in \mbb R^d$, we write $\mcal G \cdot \theta = \left\{g\cdot \theta \mid g \in \mcal G\right\}$ to denote the orbit of $\theta$ under $\mcal G$. 
It is well known that the orbit induced by the induced action of a subgroup $\mcal G \leq \mcal S_d$ corresponds to a set partition of $[d]$. We denote this partition as $\PartitionByGroup{\mcal G}$.

Let $\mcal G$ be a subgroup of $\mcal S_d$ that acts on $\mbb R^d$ via the action $\phi$. 
In a symmetric linear bandit, the expected reward is invariant under the group action of $\mathcal{G}$ on $\mathcal{X}$, that is, $ f(g \cdot x) = f(x)$.
Due to the linear structure of $f$, this is equivalent to $g \cdot \theta_\star = \theta_\star$ for all $g \in \mathcal{G}$.
We assume that, \emph{while the group action $\phi$ is known to the learner, the specific subgroup $\mcal G$ is hidden and must be learned in an online manner}.



\section{Impossibility Result of Learning with General Hidden Subgroups}
\label{sec: Partition and FixedSS, imposibility}
We now show how to frame the learning problem with hidden symmetry group as the problem of model selection. 
We further analyse the structure of the collection of models, and show that no algorithm can benefit by solely knowing that $\mcal G \leq \mcal S_d$, which implies that further assumptions are required to achieve significant improvement in term of regret.

\subsection{Fixed Point Subspace and Partition}
The analysis of learning with hidden subgroup requires a group-theoretic notion which is referred to as fixed-point subspaces \cite{Bodi2009_FixedPointSubspaceProperties}.
As we shall explain promptly, there is a tight connection between the collection of fixed-point subspaces and set partitions. 

\paragraph{Fixed-point subspaces.}
For a subset $\mcal X \subseteq \mbb R^d$, denote $\FixedSSSubGroup{\mcal G}(\mcal X) := \CurlyBr{x\in  \mcal X \mid g\cdot x = x,\; \forall g\in \mcal G}$
as the fixed-point subspace of $\mcal G$; and $\FixedSS(\mcal X) := \CurlyBr{\FixedSSSubGroup{\mcal G}(\mcal X) \mid \mcal G \leq \mcal S_d}$
as the collection of all fixed-point subspaces of all subgroups of $\mcal S_d$.
We simply write $\FixedSS = \FixedSS(\mbb R^d)$ and $\FixedSSSubGroup{\mcal G} = \FixedSSSubGroup{\mcal G}(\mbb R^d)$ for brevity.

\paragraph{Set partition.}
Given $d \in \mathbb{N}^+$, we denote $\PartitionSet_d$ as the set of all partitions of $[d]$.
Let $\PartitionSet_{d,k}$ as the set of all partitions of $[d]$ with exactly $k$ classes, and $\PartitionSet_{d,\leq k}$ be the set of all partitions of $[d]$ with at most $k$ classes.
The number of set partitions with $k$ classes $|\PartitionSet_{d,k}|$ is known as the Stirling number of the second kind, and $|\PartitionSet_d|$ is known as Bell number.

\subsection{Impossibility Result}

\paragraph{Problem with known symmetry.} Before discussing the problem of hidden symmetry, let us explain why the learner with an exact knowledge of $\mcal G$ can trivially achieve smaller regret.
The reason is that $\theta_\star \in \FixedSSSubGroup{\mcal G}$ by the assumption that $\theta_\star$ is invariant w.r.t the action of group $\mcal G$.
If $\mcal G$ is known in advance, the learner can restrict the support of $\theta_\star$ in $\mrm{Fix}_{\mcal G}$, and  immediately obtains that the regret scales with $\mrm{dim}(\FixedSSSubGroup{\mcal G})$ instead of $d$, which can be significantly smaller (e.g., if $\mcal G = \mcal S_d$, then $\mrm{dim}(\FixedSSSubGroup{\mcal G}) = 1$).

For any subgroup, there exists a fixed point subspace, and some subgroups may share the same fixed point subspace. 
Therefore, instead of constructing a collection of subgroups, one can create a smaller collection of models using the collection of fixed point subspaces.
As $\mcal G$ is hidden, one must learn $\FixedSSSubGroup{\mcal G}$ within the set of candidates $\FixedSS$, leading to the formulation of the model selection. 

\paragraph{From the setting with hidden subgroup to the setting with hidden set partition.} Now, we discuss the structure of the collection of models $\FixedSS$. First, we show the equivalent structure between the collection of fixed point subspaces and the set partitions as follows.
\medskip
\begin{restatable}{proposition}{PropBijectionFixedSSPartition}\label{prop: Bijection between Subgroup and FixedSS}
    There is a bijection $\BijectionParFixSS$ between $\PartitionSet_d$ and $\FixedSS$.
\end{restatable}
As there is a bijection between $\PartitionSet_d$ and $\FixedSS$, we can count the number of subspaces of each dimension $k$ explicitly using the following. 

\medskip

\begin{proposition}[\cite{Bodi2009_FixedPointSubspaceProperties}'s Theorem 14]\label{prop: dimension of fixed-point subspaces}
    Given a subgroup $\Gamma \leq \mcal S_d$ and its fixed-point subspace $\FixedSSSubGroup{\Gamma}$, suppose that $\PartitionByGroup{ \Gamma}$ partitions $[d]$ into $k$ classes, then $\mrm{dim}(\FixedSSSubGroup{\Gamma}) = k$.
\end{proposition}
By Proposition \ref{prop: dimension of fixed-point subspaces}, we have that the number of subspaces of dimension $k$ in $\FixedSS$ is exactly the number of set partitions with $k$-classes. 
Suppose that the learner knows that the orbit under action of $\mcal G$ partitions the index of $\theta_\star$ into 2 equivalent classes that is, $\mrm{dim}(\FixedSSSubGroup{\mcal G}) = 2$.
The learner cannot get any reduction in terms of regret. 
\medskip

\begin{restatable}{proposition}{PropLowerBoundFullSetPartition} \label{prop: Lower Bound Full Set Partition}
    Assume that the action set is the unit cube $\mcal X = \{x \in \mbb R^d \mid \|x\|_\infty \leq 1\}$, and $f$ is invariant w.r.t. action of subgroup $\mcal G \leq \mcal S_d$, such that $\mrm{dim}(\FixedSSSubGroup{\mcal G}) = 2$. Then, the regret of any bandit algorithm is lower bounded by $\mbf R_T = \Omega(d\sqrt{T})$.
\end{restatable}

The implication of Proposition \ref{prop: Lower Bound Full Set Partition} is that even if the learner knows $\theta_\star$ lies in an extremely low-dimensional subspace within the finite pools of candidates, they \emph{still suffer a regret that scales linearly with the ambient dimension $d$}. 
This suggests that \emph{further information about the group $\mathcal{G}$} must assumed to be known in order to break this polynomial dependence on $d$ in the regret bound.

\section{The Case of Hidden Subgroups with Subexponential Size}
As indicated by Proposition \ref{prop: Lower Bound Full Set Partition}, there is no improvement in terms of regret, despite the learner having access to a collection of extremely low-dimensional fixed point subspaces. 
Therefore, we assume that the learner can access only a reasonably small subset of the collection of low-dimensional fixed point subspaces.
Let $\DimensionTrueFSS$ be the upper bound for the dimension of fixed point subspaces; that is, we know that the orbit of $\mathcal{G}$ partitions $[d]$ into at most $\DimensionTrueFSS$ classes. 
Now, let us assume that the learner knows that $\mcal G$ does not partition $[d]$ freely, but must satisfy certain constraints, that is, $\PartitionByGroup{\mcal G} \in \SubsetPartitionSet_{d,\leq \DimensionTrueFSS} \subset \PartitionSet_{d,\leq \DimensionTrueFSS}$. Here, $\SubsetPartitionSet_{d,\leq \DimensionTrueFSS}$ is a small collection of partitions with at most $\DimensionTrueFSS$ classes, which encodes the constraints on the way $\mcal G$ partitions $[d]$.
We introduce an assumption regarding the cardinality of $\SubsetPartitionSet_{d,\leq \DimensionTrueFSS}$, which is formally stated in Section \ref{sec: Regret Analysis of EMC algorithm}.
Using the Proposition \ref{prop: Bijection between Subgroup and FixedSS}, we can define the collection of fixed point subspaces
associated with the collection of partition $\SubsetPartitionSet_{d,\leq \DimensionTrueFSS}$ via the bijection $\BijectionParFixSS$ as
\[\mcal M: =\BijectionParFixSS \RoundBr{ \SubsetPartitionSet_{d,\leq \DimensionTrueFSS} } \quad \text{and} \quad  M : =|\mcal M|.\]
In addition, let us define the extension of the collection $\mcal M$ as 
$
   \olsi{ \mcal M} := \CurlyBr{\mrm{conv}\RoundBr{m \cup m'} \mid\; m,\;m' \in \mcal M},
$
where $\mrm{conv}(S)$ is the convex hull of the set $S \subset \mbb R^n$.
We have that $\olsi{ \mcal M}$ is a collection of subspaces, that is, $\mrm{conv}\RoundBr{m \cup m'}$ is indeed a subspace \citep{Blumensath2009_RIP_Concentration_UoS}.
Denote $\olsi M := |\olsi{\mcal M}|$, then
we have $\olsi M  = (M^2- M)/2$.
Moreover, if dimension of subspace in $\mcal M$ is at most $\DimensionTrueFSS$, then the dimension of subspace in $\olsi{\mcal M}$ is at most $2\DimensionTrueFSS$. 
\subsection{The \texttt{Explore-Models-then-Commit} Algorithm}
Given some $n \in [T]$, we define
\begin{equation}
     Y =  X\theta_\star +  \boldsymbol \eta,
\end{equation}
where $ Y \in \mbb R^n$, $ X = \SquareBr{x_1,\dots, x_n}^\top \in \mbb R^{n\times d}$ is the design matrix, $\theta_\star \in \mbb R^d$ is the true model; $\boldsymbol{\eta} = [\eta_1,...,\eta_n]$.
We have the information that $\theta_\star$ must be contained in some (not necessarily unique) subspace $m \in \mcal M$.
Denote by $d_m$ the dimension of $m$, we have $d_m \leq \DimensionTrueFSS$ for any $m\in \mcal M$.
Let $X_m = [\Pi_m(x_t)]_{t\in [n]}^\top,$
and $S_m$ be the column space of $X_m$, one has $\mrm{dim}(S_m) \leq d_m$.
For any $m \in \mcal M$, and given $Y$, let $\Pi_{S_m}(\cdot)$ be the projection onto $S_m$. Define 
\begin{equation}\label{eq: f_m, theta_m}
    \widehat{\boldsymbol{f}}_m := \Pi_{S_m}(Y); \quad
     \widehat\theta_m := \argmin_{\theta \in m} \Norm{Y - X\theta}^2. 
\end{equation}
Now, given $n$ data points, we can choose the model $\widehat m \in \mcal M$ that minimises the least square error
\begin{equation}\label{eq: model selction procedure}
    \widehat m \in \argmin_{m\in \mcal M} \|Y - \widehat{\boldsymbol f}_m \|_2^2.
\end{equation}
Based on the framework of model selection, we now introduce our Algorithm \ref{alg: Explore Model Selection then Commit}, \texttt{Explore-Models-then-Commit} (EMC).
Our algorithm falls into the class explore-then-commit bandit algorithms.
The exploration phase consists of $t_1$ rounds.
During this phase, one samples data independently and identically distributed (i.i.d.) from an exploratory distribution $\nu$.
After the exploration phase, one computes the solution to the model selection problem and then commits to the best arm corresponding to the chosen model.

\begin{remark}
    The key step of Algorithm \ref{alg: Explore Model Selection then Commit} that may incur significant costs is solving equation \eqref{eq: model selction procedure} (line 6). 
    Without additional information about $\mcal{M}$, one might need to enumerate all models in $\mcal{M}$ and optimize among them, which would induce a time complexity of $O(n d^{c\DimensionTrueFSS})$. 
    However, if we have more information about the partitions, e.g., if they are non-crossing or non-nesting partitions, their lattice structures can be exploited to speed up the optimization process of solving equation \eqref{eq: model selction procedure}.
    Due to space limitations, we refer readers to Appendix \ref{appendix: Efficient Greedy Algorithm} for a detailed explanation of a subroutine that leverages these lattice structures for more efficient computation.
    Additionally, Section \ref{sec: experiment} demonstrates that our Algorithm \ref{alg: Explore Model Selection then Commit}, when using the lattice search algorithm for non-crossing partitions and non-nesting partitions as a subroutine, achieves polynomial computational complexity of $O(nd^5)$ and guarantees low regret.

\end{remark}

\begin{wrapfigure}{r}{0.45\textwidth}
\vspace{-14mm}
\begin{minipage}{0.45\textwidth}
{\small
\begin{algorithm}[H]
\caption{Explore Models then Commit}
\begin{algorithmic}[1]\label{alg:estc}    
    \STATE Input: $T,\; \nu,\;t_1$
    \FOR{$t= 1, \dots, t_1$}
    \STATE Independently pull arm $x_t$ according to $\nu$ and receive a reward $y_t$.
    \ENDFOR
    \STATE $X\leftarrow [x_1,..., x_{t_1}]^\top$, $Y\leftarrow [y_t]_{t\in [t_1]}$.
    \STATE Compute $\widehat m$ as \eqref{eq: model selction procedure}.
    \STATE Compute $\widehat\theta_{t_1}$ as \eqref{eq: f_m, theta_m} corresponding to $\widehat m$.
    \FOR{$t=t_1+1$ to $T$}
	\STATE Take greedy actions:
            \[x_t = \argmin_{x\in \mcal X}\AngleBr{ \widehat{\theta}_{t_1}, x}.\]
    \ENDFOR
\end{algorithmic}
\label{alg: Explore Model Selection then Commit}
\end{algorithm}
\small}
\end{minipage}
\vspace{-6mm}
\end{wrapfigure}


\vspace{4mm}
\subsection{Regret Analysis} \label{sec: Regret Analysis of EMC algorithm}
The regret analysis of Algorithm \ref{alg: Explore Model Selection then Commit} uses
results from the Gaussian model selection literature \citep{Lucien_GaussianModelSelection,Giraud2021_HighDimStatBook} as a basis. 
As such, we first state the assumptions that are common in the Gaussian model selection literature on the collection of models $\mcal M$ and the set of arms $\mcal X$ (Section~\ref{subsubsec: assumptions}).
We then provide our main analysis in Section~\ref{subsubsec: main regret}, highlighting the key technical novelties of our approach.

\subsubsection{Assumptions}
\label{subsubsec: assumptions}
Recall that due to our lower bound in Proposition \ref{prop: Lower Bound Full Set Partition}, further assumptions are required on the collection of fixed-point subspaces to achieve a reduction in terms of regret. 
As suggested by the model selection literature \citep{Kassraie2023_FeatureSelection_LinBan_LogM_Sparsity,Moradipari2021_LinearBanditRepresentationLearning}, one can achieve regret in terms of $\log(M)$ for a collection of $M$ models. 
Adopting this idea, we make the following assumption regarding the number of potential fixed-point subspaces and the set of arms. 
\medskip

\begin{assumption}[\textbf{Sub-exponential number of partitions}]\label{assp: sub-exponential models}
    The partition corresponding to $\mcal G$ belongs to a small subclass of partitions $\SubsetPartitionSet_{d,\leq \DimensionTrueFSS }\subset \PartitionSet_{d,\leq \DimensionTrueFSS}$. 
    In particular, $\pi_\mcal G \in \SubsetPartitionSet_{d,d_\star}, \; \text{for some $d_\star\leq \DimensionTrueFSS$} $, and for each $k \in [\DimensionTrueFSS]$, there exists a constant $c>0$, such that $|\SubsetPartitionSet_{d,k}| \leq O(d^{ck})$.
\end{assumption}

\medskip

\begin{assumption}[\textbf{Bounded set of arms}]\label{assp: Cube-like bounded set of arms}
    There are positive numbers $\DiamX, \RewardMax$, such that, for all $x \in \mcal X$ and $m\in \olsi {\mcal M}$, $\Norm{\Pi_m(x)}^2 \leq \DiamX$, and $|\AngleBr{x,\theta_\star}| \leq \RewardMax$.
\end{assumption}

As a consequence of Assumption \ref{assp: sub-exponential models}, the cardinality of the collection of fixed point subspaces is not too large, particularly, $M = O(d^{c\DimensionTrueFSS})$. 
First, we note that this class includes interval partitions, a structure equivalent to sparsity as a \textit{strict} subset, as explained below.
\medskip

\begin{remark}[\textbf{Equivalence between sparsity and interval partition}]\label{example: Interval Partition}
A set partition of $[d]$ is an interval partition or partition of interval if its parts are interval.
We denote $\IntervalPartition_{d}$ as the collection of all interval partition of $d$.
$\IntervalPartition_d$ admits a Boolean lattice of order $2^{d-1}$, making it equivalent to the sparsity structure in $d-1$ dimensions. Specifically, consider the set of entries of parameters $\varphi \in \mathbb{R}^{d}$ with a linear order, that is, $\varphi_1 \geq \varphi_2 \geq \dots \geq \varphi_d$.
Then define the variable $\theta \in \mathbb{R}^{d-1}$ such that $\theta_i = (\varphi_{i} - \varphi_{i+1})$. 
Each interval partition on the entries of $\varphi$ will determine a unique sparse pattern of $\theta$.
Therefore, it is clear that the cardinality of the set of interval partition with $d_0$ classes is bounded as $|\IntervalPartition_{d,\leq d_0}| = O(d^{d_0})$.
Moreover, as a result, symmetric linear bandit is strictly harder than sparse bandit and inherits all the computational complexity challenges of sparse linear bandit \footnote{Upon careful revision, we noticed that the lower-bound argument for the symmetric linear bandit with subexponential class in the NeurIPS submission version, which was adapted from the sparse linear bandit literature, is incorrect. We conjecture that the correct lower bound for the symmetric linear bandit still remains $\Omega\left(C_{\min}(\mathcal X)^{-1/3} d_0^{2/3} T^{2/3}\right)$; however, establishing this result does not appear to follow from a straightforward adaptation of existing sparse linear bandit lower bounds.}
, including the \textit{NP-hardness} of computational complexity.
\end{remark}

\begin{wrapfigure}{t!}{0.4\textwidth}
\vspace{-6mm}  
    \centering
    \includegraphics[width=\linewidth]{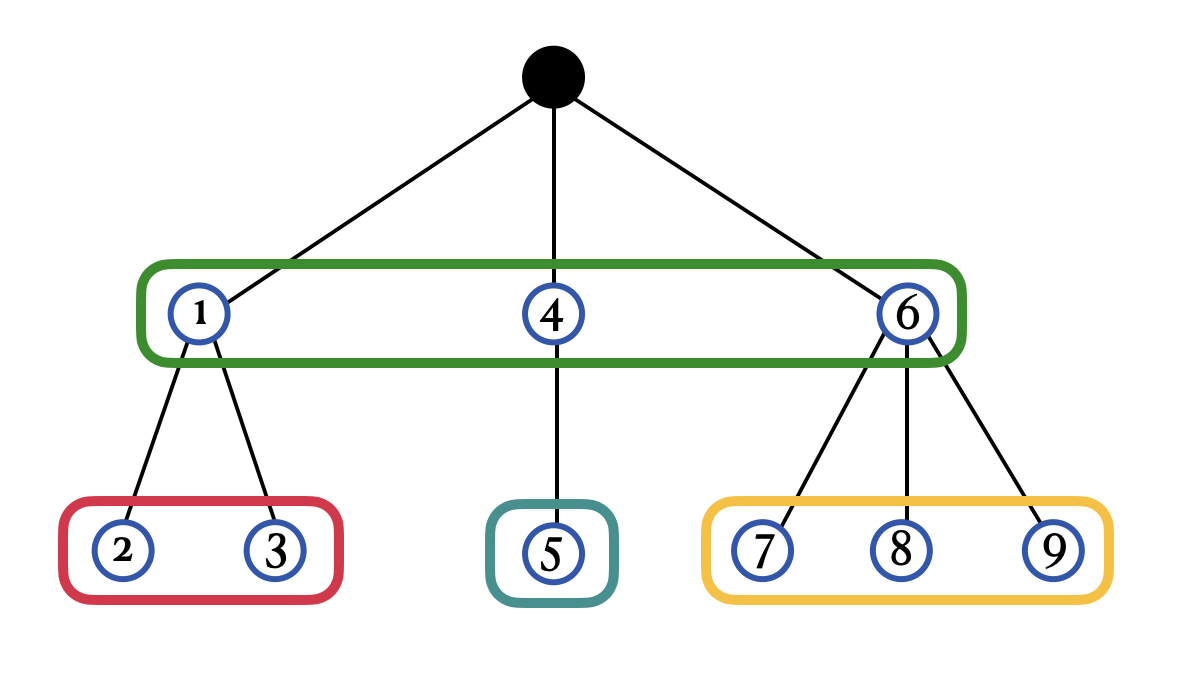}
    \caption{Partition that respects the underlying ordered tree.}
    \label{fig: partition with ordered tree}
\vspace{-8mm}
\end{wrapfigure}
Apart from sparsity, class of partitions with sub-exponential size also naturally appears when there is a hierarchical structure on the set $[d]$, and the partitioning needs to respect this hierarchical structure.
A partition that respects an ordered tree groups the children of the same node into a single equivalence class, for example, see Figure \ref{fig: partition with ordered tree}.
It is shown in \cite{Dershowitz1986_PartitionWithTree} and the cardinality of the set of partitions that respect ordered trees  is sub-exponential. 
Furthermore, as shown in \cite{Dershowitz1986_PartitionWithTree} there is a bijection between partitions that respect ordered trees and the set of non-crossing partitions.

A real-life example that meets these assumptions is  the subcontractor example in the introduction: A hierarchical structure may exist, where a hired subcontractor can further subcontract parts of the work to others.
A tree represents the hierarchical order among subcontractors, where subcontractors hired by another contractor can be grouped into one class.
Further real-life examples of non-crossing partitions and other structured partitions that satisfy sub-exponential cardinality, such as non-nesting partitions and pattern-avoidance partitions \cite{Mansour2012_SetPartitionBook}, can be found in Appendix \ref{appendix: Practical Examples of Partitions with Subexponential-Size}.



Next, similar as \cite{Hao2020_SparseLinBandit_PoorRegime}, we define the exploratory distribution as follows.
\medskip

\begin{definition}[\textbf{Exploratory distribution}]\label{def: Exploratory Distribution}
The exploratory distribution $\nu \in \Delta(\mcal X)$ is solution of the following optimisation problem 
\begin{equation}
    \nu = \argmax_{\omega \in \Delta(\mcal X)} \lambda_{\min}\RoundBr{ \mbb E_{x\sim \omega}[x x^\top]},\quad V:= \mbb E_{x\sim \nu}[x x^\top],\quad C_{\min}(\mcal X):= \lambda_{\min}(V). 
\end{equation}
\end{definition}




\subsubsection{Main Regret Upper Bound Result}
\label{subsubsec: main regret}

We now state the main result of the regret upper bound for Algorithm \ref{alg: Explore Model Selection then Commit}. 
\medskip

\begin{restatable}[\textbf{Regret upper bound}]{theorem}{ThmRegretUpperBoundEMSC} \label{Thm: Regret Upper Bound of EMSC}
    Suppose the Assumptions \ref{assp: sub-exponential models}, \ref{assp: Cube-like bounded set of arms} hold. 
    With the choice of $t_1 = \RewardMax^{-\frac{2}{3}} \sigma^{\frac{2}{3}} C_{\min}^{-\frac{1}{3}}(\mcal X) \DiamX^{\frac{1}{3}} \DimensionTrueFSS^{\frac{1}{3}}  T^{\frac{2}{3}} (\log(dT))^{\frac{1}{3}}$, then the regret of Algorithm \ref{alg: Explore Model Selection then Commit} is upper bounded as
    \begin{equation}
        \mbf R_T = O\RoundBr{\RewardMax^{\frac{1}{3}} \sigma^{\frac{2}{3}}  C_{\min}^{-\frac{1}{3}}(\mcal X)  \DiamX^{\frac{1}{3}} \DimensionTrueFSS^{\frac{1}{3}}  T^{\frac{2}{3}}  (\log(dT))^{\frac{1}{3}}}
    \end{equation}
\end{restatable}
We note that when $K_x = O(\DimensionTrueFSS)$, as in the case of sparsity, our upper bound is $\Tilde O\RoundBr{C_{\min}^{-\frac{1}{3}}(\mcal X) \DimensionTrueFSS^{\frac{2}{3}}  T^{\frac{2}{3}}}$.

The main idea is to bound the risk error after exploration rounds, as stated in the following lemma which implies the regret bound after standard manipulations. 
\medskip

\begin{restatable}{lemma}{LemExplorationError}\label{Lem: Exploration Error}
    Suppose the Assumptions \ref{assp: sub-exponential models}, \ref{assp: Cube-like bounded set of arms} hold. 
    For $t_1 = \Omega(\DiamX^2 \DimensionTrueFSS C_{\min}^{-2}(\mcal X)\log(d/\delta))$, with probability at least $1-\delta$, one has the estimate
\begin{equation}
    \Norm{\theta_\star - \widehat \theta_{t_1}} = O\RoundBr{\sqrt{\frac{\sigma^2\DimensionTrueFSS\log(d/\delta)}{C_{\min}(\mcal X)t_1}}}.
\end{equation}    

\end{restatable}
\medskip

\begin{remark}[\textbf{Non-triviality of Lemma~\ref{Lem: Exploration Error}}]
At the first glance, it seems that we can cast the problem of learning with a collection of $M$ subspaces into a model selection problem in linear bandit with $M$ features. This leads to a question: \emph{Can we apply the model selection framework based on model aggregation in \cite{Moradipari2021_LinearBanditRepresentationLearning, Kassraie2023_FeatureSelection_LinBan_LogM_Sparsity} to our case?}

First, let us explain how to cast our problem into a model selection problem in linear bandit. 
For each subspace $m$, let $\Phi_m: \mbb R^d \rightarrow \mbb R^{\DimensionTrueFSS}$ be the feature map that computes the image of the projection $\Pi_m$ with respect to the orthogonal basis of subspace $m$. Thus, we then have a collection of $M$ features $\{\Phi_m\}_{m \in \mathcal{M}}$. 
Consider the algorithm introduced in \cite{Kassraie2023_FeatureSelection_LinBan_LogM_Sparsity}, which concatenates the feature maps into $\boldsymbol{\Phi}(x) = [\Phi_1(x), \ldots, \Phi_M(x)] \in \mathbb{R}^{M d_0}$, and the regret bound depends on $\|\boldsymbol{\Phi}(x)\|_2$.

However, in our case where $\|\Phi_m(x)\|_2 < 1$, we can only bound $\|\boldsymbol{\Phi}(x)\|_2 \leq \sqrt{M}$, which leads to a $\sqrt{M}$ dependence on regret, if we use their algorithm.
Regarding \cite{Moradipari2021_LinearBanditRepresentationLearning}, their algorithm aggregates the predictions among models for each arm, and based on that, they compute the distribution for choosing each arm. 
This leads to the regret scaling with the number of arms $K$, which is not feasible in our case when $K = \infty$.

We note that the similarity with the model selection technique in \cite{Kassraie2023_FeatureSelection_LinBan_LogM_Sparsity,Moradipari2021_LinearBanditRepresentationLearning} is that they use model aggregation among $\mathcal{M}$ to bound the prediction error $\sum_{t=1}^T \left<x_t,\widehat{\theta} - \theta_\star\right>^2$, but this does not necessarily guarantee the risk error $\|\widehat{\theta} - \theta_\star\|_2$.
The reason is that, although model aggregation can guarantee a small prediction error, it imposes no restriction on the estimator $\widehat{\theta}$, which limits its ability to leverage the further benign property of designed matrix $X$.
Instead of model aggregation, our algorithm explicitly picks the best model from the pool $\mathcal{M}$, ensuring the following two properties: (1) The prediction error is small, similar to model aggregation; and (2) We can guarantee that $\widehat{\theta}$ lies in one of the subspaces of $\mathcal{M}$.
The second property gives us control over $(\widehat{\theta} - \theta_\star)$ by ensuring it lies in at most $M^2$ subspaces. 
Then, exploiting the restricted isometry property (see Definition~\ref{def:RIP}) of designed matrix $X$, we can guarantee that with $O(\log(M))$ exploratory samples, we can bound the risk error $\|\widehat{\theta} - \theta_\star\|_2$. This is crucial for eliminating polynomial dependence on $M$ and the number of arms $K$.

\end{remark}

\paragraph{Proof sketch of Lemma \ref{Lem: Exploration Error}.} We provide a proof sketch here and defer their full proof to Appendix \ref{appendix: Regret Reduction with subexponential-size collection}. Our proof borrows techniques from Gaussian model selection \citep{Giraud2021_HighDimStatBook} and the compressed sensing literature \citep{Blumensath2009_RIP_Concentration_UoS}. 
There are two steps to bound the risk error as in Lemma \ref{Lem: Exploration Error}:
\paragraph{Step 1 - Bounding the prediction error.} We can bound the prediction error $\Norm{X\theta_\star-X\widehat{\theta}_{t_1}}$ using the Gaussian model selection technique \citep{Giraud2021_HighDimStatBook} as follows. 
\medskip

\begin{restatable}{proposition}{PropModelSelectionErrorHighProb}\label{prop: PropModelSelectionErrorHighProb}
Let $\boldsymbol{f}_\star = X\theta_\star$. For the choice of $\widehat{\boldsymbol f}_{\widehat m}$ as in Eqn.~\eqref{eq: f_m, theta_m} \& Eqn.~\eqref{eq: model selction procedure}, with probability at least $1-\delta$, there exists a constant $C >1$ such that
    \begin{equation}
        \Norm{\widehat{\boldsymbol f}_{\widehat m} - \boldsymbol{f}_\star}^2 
         \leq  C\sigma^2\log\RoundBr{M\delta^{-1}}.
    \end{equation}
\end{restatable}

\paragraph{Step 2 - Bounding the risk error from prediction error.}
To bound the risk error from the prediction error, we invoke the restricted isometry property on the union of subspaces of a sub-Gaussian random matrix as in \cite{Blumensath2009_RIP_Concentration_UoS}. 
Note that $\widehat\theta$ and $\theta_\star$ can belong to two different subspaces of $\mcal M$, and $\widehat\theta-\theta_\star$ may not lie in any subspace of $\mcal M$, but in $\olsi{\mcal M}$. 
An important property of the design matrix $X$, which allows one to recover $\theta_\star$ with the knowledge that $\theta_\star$ is in a subspace $m\in \mcal M$, can be captured by the following notion of restricted isometry property (RIP): 
\medskip

\begin{definition}[\textbf{Restricted isometry property}]
\label{def:RIP}
    For any matrix $X$, any collection of subspaces $\olsi{\mcal M}$ and any $\theta \in m \in \olsi{\mcal M}$, we define $\olsi{\mcal M}$-restricted isometry constant $\delta_{\olsi{\mcal M}}(X)$ to be the smallest quantity such that
    \begin{equation}
        (1-\delta_{\olsi{\mcal M}}(X)){\Norm{\theta}^2} \leq \Norm{X\theta}^2 \leq (1+\delta_{\olsi{\mcal M}}(X)) {\Norm{\theta}^2}.
    \end{equation}
\end{definition}

We have the minimum number of samples required so that a random matrix $X$ satisfies RIP for a given constant with high probability as Proposition \ref{prop: Concentration of RIP condition with arbitrary distribution} below. Then, Lemma \ref{Lem: Exploration Error} is followed by combining Proposition \ref{prop: Concentration of RIP condition with arbitrary distribution} and Proposition \ref{prop: PropModelSelectionErrorHighProb}.


\medskip

\begin{restatable}{proposition}{PropConcentrationRIP}\label{prop: Concentration of RIP condition with arbitrary distribution}
    Let $X = [x_t]_{t\in [n]}$, where $x_t$'s are is i.i.d. drawn from $\nu$, and 
    let 
    $
        n =\Omega\RoundBr{ C_{\min}^{-2}(\mcal X)\DiamX^2 \RoundBr{\log(2\olsi{M} \DimensionTrueFSS \delta^{-1})}}.
    $
    Then, with probability at least $1-\delta$, and for any $\theta_1, \theta_2$ in subspaces of $\mcal M$, one has that
    \begin{equation}
          \Norm{\theta_1 - \theta_2}^2 \leq 2C_{\min}^{-1}(\mcal X) n^{-1} \Norm{X(\theta_1 -\theta_2) }^2.
    \end{equation}
\end{restatable}




\section{Improved Regret Upper Bound with Well-Separated Partitions}

\label{subsec: well-separated}
We now show that by adding more structure (i.e., well-separatedness) to the setting, we can further improve the regret upper bound to $O(\sqrt{T})$.
In particular, we will introduce the notion of well-separatedness in this section, and show that this notion can lead to improved (i.e., $O(\sqrt{T})$) regret bounds.

\begin{wrapfigure}{t!}{0.45\textwidth}
\vspace{-8mm}  
\begin{minipage}{0.45\textwidth}
{\small
\begin{algorithm}[H]
\caption{Exploring Model then Commit with well-separated partition}
\begin{algorithmic}[1]\label{alg:estc true model} 
    \STATE Input: $T,\; \nu, \; t_2$
    \FOR{$t= 1, \cdots, t_2$}
    \STATE Independently pull arm $x_t$ according to $\nu$ and receive a reward $y_t$.
    \ENDFOR
    \STATE $X\leftarrow [x_1,..., x_{t_1}]^\top$, $Y\leftarrow [y_t]_{t\in [t_1]}$.
    \STATE Compute $\widehat m$ as \eqref{eq: model selction procedure}.
    \FOR{$t=t_2+1$ to $T$}
	\STATE Playing \text{OFUL} algorithm \cite{Abbasi2011_OFUL} on $\widehat m$.
    \ENDFOR
\end{algorithmic}
\end{algorithm}
}
\end{minipage}
\vspace{-4mm}
\end{wrapfigure}

For each partition $p \in \PartitionSet_d$, there is a unique equivalence relation on $[d]$ corresponding to $p$. 
Denote by $\overset{p}{\sim}$ the equivalence relation corresponding to $p$. 
Next, we define well-separatedness. 
\medskip

\begin{assumption}[\textbf{Well-separated partitioning}] \label{assp: well distinguised partition}
Given the true subgroup $\mcal G$, and the corresponding partition $\PartitionByGroup{\mcal G}$. For all 
$(i,j)$ such that $ i\overset{\PartitionByGroup{\mcal G}}\nsim j $,  it holds that $|\theta_{\star,i} -\theta_{\star,j}| \geq \varepsilon_0$, for some $\varepsilon_0 > 0$.
\end{assumption}

The implication of Assumption \ref{assp: well distinguised partition} is that the projection of $\theta_\star$ to any subspace $m \in \mathcal{M}$ not containing $\theta_\star$ will cause some bias in the estimation error. 
In particular, one can show that for any $m \in \mcal M$ such that $\theta_\star \notin m$, it holds that
\begin{equation}
    \Norm{\theta_\star - \Pi_m(\theta_\star)}^2 \geq \varepsilon_0^2/2.
\end{equation}
We now show that under the Assumption \ref{assp: well distinguised partition}, after the exploring phase, the algorithm returns a true fixed-point subspace $\widehat m \ni \theta_\star$ with high probability. 
\medskip

\begin{restatable}{theorem}{ThmRegretUpperBoundWellDistinguishedPartition} \label{Thm: Regret Upper Bound of LTMC}
    Suppose the Assumptions \ref{assp: sub-exponential models}, \ref{assp: Cube-like bounded set of arms}, \ref{assp: well distinguised partition} hold.
    Let $t_2 = \Omega\RoundBr{\frac{\sigma^2 \DiamX^2 \DimensionTrueFSS\log(dT)}{C_{\min}^2(\mcal X)\varepsilon_0^2}}$. Then, Algorithm \ref{alg:estc true model} returns $\widehat m \ni \theta_\star$ with probability at least $1-1/T$, and its regret is upper bounded as 
        \begin{equation}
        \mbf R_T = O\RoundBr{\frac{\RewardMax\sigma^2 \DiamX^2 \DimensionTrueFSS\log(dT)}{C_{\min}^2(\mcal X)\varepsilon_0^2}  + \sigma d_0\sqrt{T\log(\DiamX T)}}.
    \end{equation}
\end{restatable}
That is, if the separating constant $\varepsilon_0$ is known in advance and $\varepsilon_0 \geq T^{-1/4}$, then we can achieve $O(d_0\sqrt{T}\log(\DiamX T))$ regret upper bound. 

\medskip

\begin{remark}
A weakness of Algorithm \ref{alg:estc true model} is that 
without knowing that $\varepsilon_0 \geq T^{-1/4}$ is true a priori, there may be possible mis-specification error, which leads to  linear regret if one applies the algorithm naively.
On the other hand, Algorithm \ref{alg: Explore Model Selection then Commit} can always achieve regret $O(T^{2/3})$ in the worst case. As such, the following question arises: \textit{Does there exist an algorithm that, without the knowledge of $\varepsilon_0$, can achieve regret $O(\sqrt{T})$ whenever $\varepsilon_0 \geq T^{-1/4}$, but guarantees the worst-case regret of $O(T^{2/3})$?}
Toward answering this question, we propose a simple method which has $O(\sqrt{T})$ regret whenever the separating constant is large, and enjoys a worst-case regret guarantee of $O(T^{3/4})$ (slightly worse than $O(T^{2/3})$).
We refer the reader to Appendix \ref{appendix: Adapt with separating constant} for a detailed description of the algorithm, its regret bound and further discussion.
\end{remark}
%

\section{Experiment}
\label{sec: experiment}
To illustrate the performance of our algorithm, we conduct simulations where the entries of $\theta_\star$ satisfy three cases: sparsity, non-crossing partitions and non-nesting partitions. 
We refer readers to Appendix \ref{appendix: Structured Partition and Computation } for a more formal description of non-crossing partitions, non-nesting partitions, and why the interval partition (i.e., the partition structure equivalent to sparsity) is a strict subset of both non-crossing and non-nesting partitions.
Since sparsity is equivalent to a strict subset of non-crossing and non-nesting partitions, we compare our Algorithm \ref{alg: Explore Model Selection then Commit} with the sparse-bandit ESTC algorithm proposed in \cite{Hao2020_SparseLinBandit_PoorRegime} as a benchmark in all environments.
The set of arms $\mcal X$ is $\sqrt{d} \mbb S^{d-1}$, $\sigma = 0.1$, and $d=100$, $d_0=15$. The ground-truth sparse patterns, partitions and $\theta_\star$ are randomized before each simulation. 

The regret of both algorithms is shown in Figure \ref{fig: simulation result for Noncrossing partition}, which indicates that our algorithm performs competitively in the sparsity case and significantly outperforms the sparse-bandit algorithm in cases of non-crossing and non-nesting partitions. 
Due to space limitations, we refer the reader to Appendix \ref{appendix: experiment} for a detailed description of the experiments, including how we applied the sparse-bandit algorithm in the cases of non-crossing and non-nesting partitions, and how we ran Algorithm \ref{alg: Explore Model Selection then Commit} in the case of sparsity. 
Additionally, we explain how we exploited the particular structure of non-crossing and non-nesting partitions to enable efficient computation in Appendix \ref{appendix: experiment}.

\begin{figure}[h!] 
    \centering
    \includegraphics[width=1\linewidth]{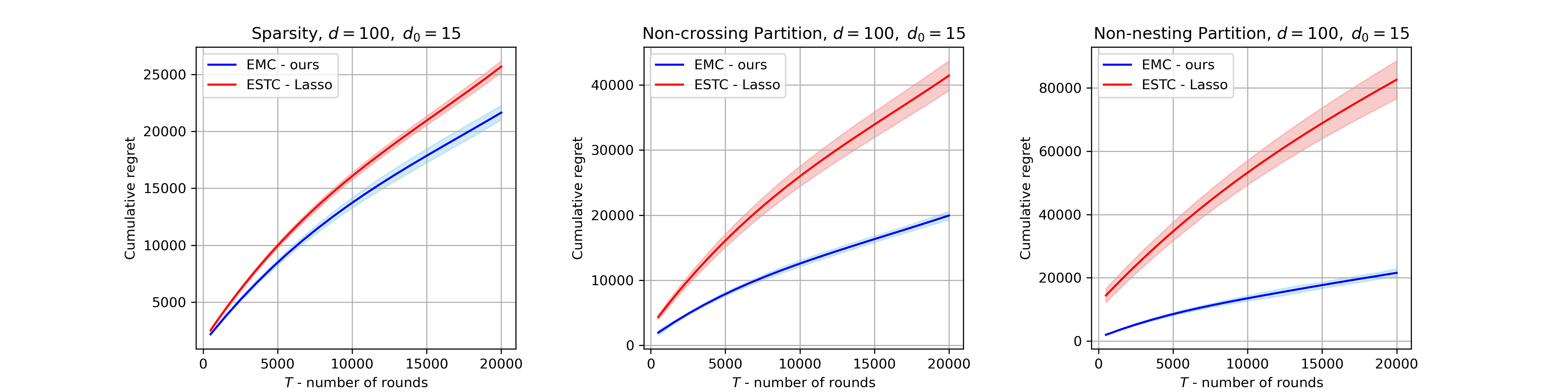}
    \caption{Regret of EMC (Algorithm \ref{alg: Explore Model Selection then Commit}) and of ESTC proposed in \cite{Hao2020_SparseLinBandit_PoorRegime}, in cases of sparsity, non-crossing partitions, and non-nesting partitions.}
    \label{fig: simulation result for Noncrossing partition}
\end{figure}


\section{Conclusion and Future Work}


In this paper, we study symmetric linear stochastic bandits in high dimensions, where the linear reward function is invariant with respect to some hidden subgroup $\mathcal{G} \leq \mcal S_d$. We first prove that no algorithm can gain any advantage solely by knowing $\mathcal{G} \leq S_d$.
Given this, we introduce a cardinality condition on the hidden subgroup $\mcal G$, allowing the learner to overcome the curse of dimensionality.
Under this condition, we propose novel model selection algorithms that achieve regrets of $\tilde O(\DimensionTrueFSS^{2/3}T^{2/3})$ and $\tilde O(d_0\sqrt{T})$ with an additional assumption on the well-separated partition.
For future work, we will explore convex relaxation techniques for efficient computation, leveraging specific structures of symmetries.

\newpage
\bibliographystyle{plainnat}
\bibliography{ref}

\newpage
\appendix
\renewcommand*\contentsname{Contents of Appendix}
\addtocontents{toc}{\protect\setcounter{tocdepth}{2}}
\doublespacing
\tableofcontents
\singlespacing

\section*{Additional notations}
For any subset $S \subseteq \mbb R^d$, and matrix $W \in \mbb R^{d\times d}$, we write
$
W(S) := \{Wx \mid x\in S\}.
$

\section{Impossibility Result of Learning with General Hidden Subgroups}
\label{appendix: Impossibility of learning with general unknown subgroups}

To prove Proposition 1, we need to prove Proposition~\ref{prop: From subgroup to partition} and use Proposition~\ref{prop: construction of fixed-point subspace}.
\medskip

\begin{proposition}\label{prop: From subgroup to partition}
    For each $\Gamma \leq \mcal S_d$ acting naturally on $[d]$, the orbit of $\Gamma$ on $[d]$ forms a unique partition of $[d]$. 
    Moreover, for each partition $\rho \in \mcal \PartitionSet_d$, there exists at least one $\Gamma \leq \mcal S_d$ such that its orbit on $[d]$ under natural action is exactly $\rho$.
\end{proposition}

\begin{proof}
    The first claim is obvious by the property of orbit, that is, orbit consists of non-empty and disjoint subsets of $[d]$, whose union is $[d]$.

    We prove the second claim.
    Let $\rho = \{\rho_i\}_{i\in I}$, where $I$ is the index of partition; note that $\rho_i$ is nonempty and mutually disjoint.
    Define a group as follows
    \begin{equation}
        \Gamma_i = \CurlyBr{f : \rho_i \rightarrow \rho_i \mid \text{$f$ is bijective}};
    \end{equation}
    It is clear that $\Gamma_i$ is a group under function composition.
    Now, define the product group
    \[
    \Gamma := \prod_{i\in I} \Gamma_i,
    \]
    and the action $\psi:\; \Gamma \times [d]   \rightarrow [d]$ such that 
    \begin{equation}
        \psi\RoundBr{(f_i)_{i\in I}, x} := f_j(x), \quad \text{for $ \rho_j\ni x$}.
    \end{equation}
    Therefore, it is clear that $(f_i)_{i\in I}$ is a bijection from $[d]$ onto itself, hence, $\Gamma$ is a subgroup of $\mcal S_d$.
\end{proof}

Let $E = (e_i)_{i\in [d]}$ be the standard basis. Given that group $\mcal G$ partition $[d]$ into $k$ disjoint orbits, $\mcal G$ also partition $E$ into $k$ disjoint orbits corresponding to its action $\phi$ on $\mbb R^d$, that is,  
\[E = \bigcup_{i=1}^k E_i.\]
Let $V_i = \mrm{Span} (E_i)$, then, one has the following.
\medskip

\begin{proposition}[\cite{Bodi2009_FixedPointSubspaceProperties}'s Theorem 14] \label{prop: construction of fixed-point subspace}
    \begin{equation}
        \FixedSSSubGroup{\mcal G}(\mbb R^d) =  \bigoplus_{i=1}^k  \FixedSSSubGroup{\mcal G}(V_i).
    \end{equation}
\end{proposition}
We now state the proof of Proposition \ref{prop: Bijection between Subgroup and FixedSS} as a corollary of Proposition \ref{prop: From subgroup to partition} and Proposition \ref{prop: construction of fixed-point subspace}.
\medskip

\PropBijectionFixedSSPartition*
\begin{proof}
    \textbf{First}, we show that for each set partition of $[d]$, there exists a unique $H \in \FixedSS(\mbb R^d)$.
    This is straightforward due to the fact that, given a set partition $P$, the partition of basis $(E_i)_{i\in I}$ is unique by definition, so as $V_i = \mrm{span}(E_i)$.
    As a result, let $H =  \bigoplus_{i=1}^k \FixedSSSubGroup{\Gamma}(V_i)$, by Proposition \ref{prop: construction of fixed-point subspace}, $H \in \FixedSS(\mbb R^d)$.
    As $(V_i)_{i\in I}$ is unique, $H$ is unique.
    
    \textbf{Second}, we show that for each $H \in \FixedSS(\mbb R^d)$, there is a unique set partition/equivalent relation $P$.
    Denote $\overset{\mrm P}\sim$ as an equivalent relation under the set partition $P$.
    Suppose there are two different set partitions $P, \; Q$, there must be two set element $p\; q$ such that $p \overset{\mrm P}\sim q$ under $P$, and $p \overset{\mrm Q}\nsim q$ under $Q$.
    Denote $H_P,\; H_Q$ as the subspaces defined by $P$ and $Q$ respectively.
    As $p \overset{\mrm Q}\nsim q$, there must exist a point $x \in H_Q$ such that $x_p \neq x_q$. However, $x_p = x_q$ for all $x\in H_P$. 
    Hence, $H_P$ cannot be the same as $H_Q$.
\end{proof}

\medskip
\PropLowerBoundFullSetPartition*
\begin{proof}[Proof]
    Suppose there is a collection of model parameter $\Theta = \{-\varepsilon,\;\varepsilon\}^d$, for some $\varepsilon>0$.
    For each $\theta \in \Theta$, it is straightforward that since $\theta$ has two classes of indices, by Proposition \ref{prop: Bijection between Subgroup and FixedSS} and \ref{prop: dimension of fixed-point subspaces}, there must be a subgroup $\mcal G \leq \mcal S_d$ such that $\theta \in \FixedSSSubGroup{\mcal G}$ and $\dim\RoundBr{\FixedSSSubGroup{\mcal G}} = 2$. 

    Now, $\Theta$ can be used as a family of problem instances for minimax lower bound of linear bandit. 
    By Theorem 24.1 in \cite{Lattimore2020_BanditBook}, we have that with the choice of $\varepsilon = T^{-1/2}$, one has that
    \begin{equation}
       \mbf R_T \geq \frac{\exp{-2}}{\sqrt{8}}d\sqrt{T}.
    \end{equation}
\end{proof}
\section{The Case of Hidden Subgroups with Subexponential Size}
\label{appendix: Regret Reduction with subexponential-size collection}
\subsection{High Probability Prediction Error of Model Selection}
The proof of Proposition \ref{prop: PropModelSelectionErrorHighProb} uses the following concentration.
\medskip

\begin{proposition}[\cite{Giraud2021_HighDimStatBook}'s Theorem B.7]\label{prop: concentration of projection of noise}
    For some subspace $S \subset \mbb R^n$ of dimension $d_S$, with probability at least $1-\delta$, one has that
    \begin{equation}
       \left |\Norm{\Pi_{S}(\boldsymbol\eta)} - \sigma \sqrt{d_S} \right| \leq   \sigma\sqrt{2\log\RoundBr{\frac{1}{\delta}}}.
    \end{equation}
\end{proposition}

\medskip

\PropModelSelectionErrorHighProb*
\begin{proof}
Define $\mcal E_1$ as the event such that for all $m\in \mcal M$,
\begin{equation}
    \Norm{\Pi_{S_m}(\boldsymbol\eta)} \leq \sigma \sqrt{d_m} + \sigma\sqrt{2\log\RoundBr{\frac{1}{\delta_0}}}.
\end{equation}
By Proposition \ref{prop: concentration of projection of noise} and union bound, $\mcal E_1$ occurs with probability $1-M\delta_0$.
For the rest of the proof, we assume $\mcal E_1$ occurs .

Let us denote $\boldsymbol{\widehat f} := \boldsymbol{\widehat f}_{\widehat m}$.
By the model selection procedure \eqref{eq: model selction procedure}, one has that
\begin{equation}
    \|Y - \boldsymbol{\widehat f} \|^2  \leq \|Y - \boldsymbol{\widehat f}_{m_\star} \|^2 .
\end{equation}
Also, as $Y = \boldsymbol{f}_\star +\boldsymbol\eta$, one has that, for some $K>1$
\begin{equation}\label{eq: 1st expression}
\begin{aligned}   
    \Norm{\boldsymbol{\widehat f} - \boldsymbol{f}_\star}^2 \leq \Norm{\boldsymbol{f}_\star - \boldsymbol{\widehat f}_{m_\star} }^2 + 2\AngleBr{\boldsymbol\eta, \boldsymbol{f}_\star - \boldsymbol{\widehat f}_{m_\star}} + 
    K\sigma^2 d_{\widehat m} -  2\AngleBr{\boldsymbol\eta, \boldsymbol{f}_\star - \boldsymbol{\widehat f}} - K\sigma^2 d_{\widehat m}.
\end{aligned}
\end{equation}

\textbf{First}, consider the term $\Norm{\boldsymbol{f}_\star - \boldsymbol{\widehat f}_{m_\star} }^2$. 
Note that, $\Pi_{S_{m_\star}}(\boldsymbol{f}_\star) = \boldsymbol{f}_\star$, as $\boldsymbol{f}_\star \in S_{m_\star}$.
We have that, 
\begin{equation}\label{eq: First term}
    \begin{aligned}
        \Norm{\boldsymbol{f}_\star - \boldsymbol{\widehat f}_{m_\star} }^2 &= \Norm{\boldsymbol{f}_\star - \Pi_{S_{m_\star}}(\boldsymbol{f}_\star  + \boldsymbol\eta) }^2 \\
        &= \Norm{\boldsymbol{f}_\star - \boldsymbol{f}_\star}^2 - 2\AngleBr{\Pi_{S_{m_\star}}(\boldsymbol\eta), \boldsymbol{f}_\star - \boldsymbol{f}_\star}  + \Norm{\Pi_{S_{m_\star}}( \boldsymbol\eta)}^2 \\
        &\leq \RoundBr{\sigma\sqrt{d_{m_\star}} + \sigma\sqrt{2\log\RoundBr{\frac{1}{\delta_0}}}}^2 \\
        &\leq 2\sigma^2 \DimensionTrueFSS + 4\sigma^2\log\RoundBr{\frac{1}{\delta_0}}.
    \end{aligned}
\end{equation}

\textbf{Second}, consider the term $\AngleBr{\boldsymbol\eta, \boldsymbol{f}_\star - \boldsymbol{\widehat f}_{m_\star}}$.
We have that, 
\begin{equation}\label{eq: Second term}
    \begin{aligned}
        \AngleBr{\boldsymbol\eta, \boldsymbol{f}_\star - \boldsymbol{\widehat f}_{m_\star}} &=  \AngleBr{\boldsymbol\eta, \boldsymbol{f}_\star - \Pi_{S_{m_\star}}(\boldsymbol{f}_\star +\boldsymbol\eta)} \\
        &= \AngleBr{\boldsymbol\eta,\boldsymbol{f}_\star  - \boldsymbol{f}_\star} - \Norm{\Pi_{S_{m_\star}}(\boldsymbol\eta)}^2 \leq 0.
    \end{aligned}
\end{equation}

\textbf{Third}, we need to control the magnitude of the term $2\AngleBr{\boldsymbol\eta, \boldsymbol{\widehat f} - \boldsymbol{f}_\star} - K\sigma^2 d_{\widehat m}$.
We show that this term can be bounded as 
\begin{equation*}
    2\AngleBr{\boldsymbol\eta, \boldsymbol{\widehat f} - \boldsymbol{f}_\star} - K\sigma^2 d_{\widehat m} \leq a^{-1}\Norm{\boldsymbol{\widehat f} - \boldsymbol{f}_\star}^2 + O(\log(1/\delta_0)),
\end{equation*}

with probability at least $1-M\delta_0$, for any constant $a>0$.
Denote $\AngleBr{\boldsymbol{f}_\star}$ be the line that is spanned by $\boldsymbol{f}_\star$.
For each $m \in \mcal M$, let us define the subspace $\Bar{S}_{m} = S_m + \AngleBr{\boldsymbol{f}_\star}$.
Define $\Tilde S_m \subset \Bar{S}_{m}$ as the subspace that is orthogonal to $\AngleBr{\boldsymbol{f}_\star}$, that is, one can write $\Bar{S}_{m} = \Tilde S_m \bigoplus \AngleBr{\boldsymbol{f}_\star}$.
By AM-GM inequality, for some $a>0$, we have that
\begin{equation}
    \begin{aligned}
       2 \AngleBr{\boldsymbol\eta, \boldsymbol{\widehat f} - \boldsymbol{f}_\star} =  2\AngleBr{\Pi_{\bar S_{\widehat m}}(\boldsymbol\eta), \boldsymbol{\widehat f} - \boldsymbol{f}_\star} &=  2\AngleBr{\sqrt{a}\cdot \Pi_{\bar S_{\widehat m}}(\boldsymbol\eta), \frac{1}{\sqrt{a}}\left(\boldsymbol{\widehat f} - \boldsymbol{f}_\star\right)}\\
        &\leq a \Norm{\Pi_{\bar S_{\widehat m}}(\boldsymbol\eta)}^2 + a^{-1}\Norm{\boldsymbol{\widehat f} - \boldsymbol{f}_\star}^2 \\
        &= a\sigma^2 V + a\sigma^2 U_{\hat{m}} + a^{-1}\Norm{\boldsymbol{\widehat f} - \boldsymbol{f}_\star}^2 ,
    \end{aligned}
\end{equation}
where $V = \sigma^{-2}\Norm{\Pi_{\AngleBr{\boldsymbol{f}_\star}}(\boldsymbol\eta)}^2$ and $U_{\hat{m}} = \sigma^{-2}\Norm{\Pi_{\Tilde S_{\hat{m}}} (\boldsymbol\eta)}^2$.
Note that, as $\mrm{dim}(\AngleBr{\boldsymbol{f}_\star}) = 1$, with probability at least $1 - \delta_0$, one has that
\begin{equation}\label{eq: V}
    V \leq 2 + 4\log(\delta^{-1}_0).
\end{equation}

Define the event $\mcal E_2$ as the event where the above inequality holds.

Therefore, our final task is to control the quantity $a U_{\hat{m}} - K d_{\widehat m}$.
Choose $a = (K+1)/2 > 1$, one has that
\[
a U_{\hat{m}} -K d_{\widehat m} = \frac{K+1}{2}\RoundBr{U_{\hat{m}} - \frac{2K}{K+1}d_{ \widehat m} } \leq \frac{K+1}{2} \max_{m\in \mcal M}\RoundBr{U_{m} - \frac{2K}{K+1}d_{ m} }.
\]

Now, directly control the magnitude of $U_{\hat{m}} - \frac{2K}{K+1}d_{\widehat m} $ is difficult, as $\widehat m$ depends on $\boldsymbol\eta$, and their distribution might be complicated.
Instead, we will control the maximum of the above quantity $U_{m} - \frac{2K}{K+1}d_{m}$ over all $m\in \mcal M$.
Since the dimension of $\Tilde S_{m}$ is at most $d_m$, similar as the event $\mcal E_1$, we define the event $\mcal E_3$ as for all $m \in \mcal M$, we have that
\begin{equation}
    \Norm{\Pi_{\Tilde S_{m}} (\boldsymbol\eta)} \leq \sigma\sqrt{d_m} + \sigma\sqrt{2\log\RoundBr{\frac{1}{\delta_0}}}.
\end{equation}
By Proposition \ref{prop: concentration of projection of noise} and the union bound, $\mcal E_3$ occurs with probability at least $1- M\delta_0$. 
We assume that  $\mcal E_3$ occurs, then for all $m\in \mcal M$,
\begin{equation}
\begin{aligned}
       U_{m} &\leq \RoundBr{\sqrt{d_m} + \sqrt{2\log\RoundBr{\frac{1}{\delta_0}}}}^2 \\
    &\leq  \frac{2K}{K+1} d_m  + \frac{4K}{K-1}\log\RoundBr{\frac{1}{\delta_0}}.
\end{aligned}
\end{equation}
Therefore, one has that for all $m \in \mcal M$,
\begin{equation}\label{eq: Um}
\begin{aligned}
    \frac{K+1}{2}\RoundBr{U_{\hat{m}} - \frac{2K}{K+1}d_{\widehat m}}  &\leq  \frac{K+1}{2} \max_{m\in \mcal M}\RoundBr{U_m - \frac{2K}{K+1} d_m } \\
    &\leq \frac{2K(K+1)}{K-1}\log\RoundBr{\frac{1}{\delta_0}}.
\end{aligned}
\end{equation}

\textbf{Putting things together,} let $\delta = (2M+1) \delta_0 $, we assure that event $\mcal E_1 \cap \mcal E_2 \cap \mcal E_3$ occurs with probability at least $1-\delta$.
Combining \eqref{eq: First term}, \eqref{eq: Second term}, \eqref{eq: V}, \eqref{eq: Um}, with \eqref{eq: 1st expression}, one has that
\begin{equation}
\begin{aligned}
        \frac{K-1}{K+1} \Norm{\boldsymbol{\widehat f} - \boldsymbol{f}_\star}^2 &\leq 2\sigma^2 \DimensionTrueFSS+  4\sigma^2 \log\RoundBr{\frac{2M+1}{\delta}} + K\sigma^2 \DimensionTrueFSS  \\
        &+ \sigma^2(K+1)
        + 2\sigma^2(K+1)\log\RoundBr{\frac{2M+1}{\delta}} + \sigma^2\frac{2K(K+1)}{K-1}\log\RoundBr{\frac{2M+1}{\delta}}.
\end{aligned}
\end{equation}

Note that, the dominating term in the above equation is $\log\RoundBr{\frac{M}{\delta}}$, we have that, there is some constant $C > 0$.
\begin{equation}
    \Norm{\boldsymbol{\widehat f} - \boldsymbol{f}_\star}^2 \leq C \sigma^2 \log\RoundBr{\frac{M}{\delta}}.
\end{equation}
\end{proof}
\medskip

\PropConcentrationRIP*
\begin{proof}
Our proof strategy is inspired by \cite{Blumensath2009_RIP_Concentration_UoS}. 

First, we compute the subgaussian norm for the normalised distribution along fews directions $\nu$. For any $x\in \mcal X$, we have that 
\begin{equation} \label{eq: Bounded normalised exploratory distrbution}
    \begin{aligned}
        \max_{\substack{\Norm{v} = 1, \\ v \in V^{1/2}(\olsi m),\;\olsi m \in \olsi{\mcal M}} } \AngleBr{V^{-1/2}x, v}^2  
        &= \max_{\substack{\Norm{V^{1/2} \omega} = 1, \\ \omega \in \olsi m,\;\olsi m \in \olsi{\mcal M}} } \AngleBr{V^{-1/2}x, V^{1/2} \omega}^2  \\
        &= \max_{\substack{\Norm{V^{1/2} \omega} = 1, \\ \omega \in \olsi m,\;\olsi m \in \olsi{\mcal M}} } \AngleBr{x,  \omega}^2 \\
        & \leq\max_{\olsi m \in \olsi{\mcal M}} \Norm{\Pi_{\olsi m}(x)}^2 C_{\min}^{-1}(\mcal X) \leq \frac{\DiamX}{C_{\min}(\mcal X)}.
    \end{aligned}
\end{equation}
This means that, the distribution $\AngleBr{V^{-1/2}x, v}^2$ has bounded support $[0,\; \frac{\DiamX}{C_{\min}}]$ for any above direction of $v$.

Next, for any $\theta\in \olsi m$, denote $\theta_V := V^{1/2}\theta $, note that the direction $v$ corresponding to $\theta_V$ satisfies \eqref{eq: Bounded normalised exploratory distrbution}.
Using Hoeffding's inequality for bounded random variable, there is an absolute constant $c_1 >0$ such that
\begin{equation}
\begin{aligned}
    &\mrm{Pr}\CurlyBr{\bigg|\frac{1}{n}\sum_{i=1}^n \AngleBr{V^{-1/2}x_i, \theta_V}^2 - \Norm{\theta_V}^2 \bigg| \geq \epsilon \Norm{\theta_V}^2} \leq 2\exp\RoundBr{-\frac{ n c_1 C_{\min}^2(\mcal X) \epsilon^2}{\DiamX^2}} \\
    \iff & \mrm{Pr}\CurlyBr{\bigg|\frac{1}{n} \Norm{XV^{-1/2} \theta_V}^2 - \Norm{\theta_V}^2 \bigg| \geq \epsilon \Norm{\theta_V}^2} \leq 2\exp\RoundBr{-\frac{ n c_1 C_{\min}^2(\mcal X) \epsilon^2}{\DiamX^2}}.
\end{aligned}
\end{equation}

Let $Z := \frac{1}{\sqrt{n}} XV^{-1/2}$. From Lemma 5.1 in \cite{Baraniuk2008_RandomMatrixConcentration}, we know that if the above inequality holds, then 
\begin{equation}
    \RoundBr{1-\delta_{\olsi M}\RoundBr{Z} \Norm{\theta_V}}   \leq  \Norm{Z \theta_V } \leq  \RoundBr{1+\delta_{\olsi M}\RoundBr{Z} \Norm{\theta_V}},
\end{equation}
holds with probability more than
\begin{equation*}
    1-2\RoundBr{\frac{12}{\delta_{\olsi M}\RoundBr{Z}}}^{2\DimensionTrueFSS}\exp\RoundBr{-\frac{c_1 C_{\min}^2(\mcal X) n \delta_{\olsi M}^2\RoundBr{Z}}{\DiamX^2}},
\end{equation*}
for any $\theta_V \in V^{1/2}(\olsi{m})$ in a subspace $\olsi{m} \in \olsi{M}$.
Take union bound for $\olsi M$ subspaces, and let
\begin{equation*}
    \begin{aligned}
    \delta &= 2\olsi{M}  \RoundBr{\frac{2}{\delta_{\olsi M}\RoundBr{Z}}}^{2\DimensionTrueFSS}\exp\RoundBr{-\frac{c_1 C_{\min}^2(\mcal X) n \delta_{\olsi M}^2\RoundBr{Z}}{\DiamX^2}}, \\
    \iff n  &=  O\RoundBr{\frac{\DiamX^2}{\delta_{\olsi M}^2\RoundBr{Z} C_{\min}^2(\mcal X) }\RoundBr{\log(2\olsi{M}) + \DimensionTrueFSS\log\RoundBr{\frac{1}{\delta_{\olsi M}\RoundBr{Z}}}  + \log(\delta^{-1})}}. \\
    \end{aligned}
\end{equation*}
Let $\delta_{\olsi M}^2\RoundBr{Z} = 1/2$. Then, Given $n  =  O\RoundBr{\frac{\DiamX^2}{C_{\min}^2(\mcal X) }\RoundBr{\log(2\olsi{M}) + \DimensionTrueFSS  + \log(\delta^{-1})}}$, for probability at least $1-\delta$, we have that
\begin{equation*}
\begin{aligned}
     &  \frac{1}{2} \Norm{V^{1/2}\theta}^2   &&\leq  \frac{1}{n}\Norm{XV^{-1/2} V^{1/2}\theta }^2, \\
\implies & \frac{1}{2} C_{\min}(\mcal X) \Norm{\theta}^2  &&\leq \frac{1}{n} \Norm{X\theta}^2.
\end{aligned}
\end{equation*}
Let $\theta: = \theta_1 - \theta_2$, for any $\theta_1 \in m$, $\theta_2 \in m'$, for any $m, m' \in \mcal M$. We conclude the proof.
\end{proof}

\subsection{Regret Upper bound}
\medskip

\LemExplorationError*

\begin{proof}
    First, let consider Proposition \ref{prop: Concentration of RIP condition with arbitrary distribution},
    and recall that $\log(M)$  and $\log(\olsi M)$ are both $O(\DimensionTrueFSS \log(d))$.
    Therefore, for the choice of $t_1 = \Omega(\DiamX^2\DimensionTrueFSS C_{\min}^{-2}(\mcal X) \log(d/\delta))$, with probability at least $1-\delta/2$, one has that
    \begin{equation}
        C_{\min}(\mcal X)\Norm{\theta_\star - \widehat \theta_{t_1}}^2 \leq \frac{2}{t_1} \Norm{X(\theta_\star - \widehat \theta_{t_1})}^2 
    \end{equation}

    Second, by Proposition \ref{prop: PropModelSelectionErrorHighProb}, one has that with probability at least $1-\delta/2$.
    \begin{equation}
        \Norm{X(\theta_\star - \widehat \theta_{t_1})}^2 \leq c_1 \sigma^2 \DimensionTrueFSS \log\RoundBr{\frac{d}{\delta
        }},
    \end{equation}
    for some constant $c_1>0$.
    
    Therefore, putting thing together, we have that with probability at least $1-\delta$, one has that
    \begin{equation}
        \Norm{\theta_\star - \widehat \theta_{t_1}} \leq c_2 \sqrt{\frac{\sigma^2 \DimensionTrueFSS \log\RoundBr{\frac{d}{\delta
        }}}{C_{\min}(\mcal X)t_1}},
    \end{equation}
    for some constant $c_2>0$.
\end{proof}

\medskip

\ThmRegretUpperBoundEMSC*
\begin{proof}
    Let define the pesudo regret as $\widehat{\mbf R}_T = \sum_{t=1}^T \AngleBr{x_\star - x_t,\theta_\star}$.
    Denote $\olsi m \in \olsi M$ be the subspace contains $\theta_\star - \widehat \theta_{t_1}$.
We start by simple regret decomposition as follows.

    \begin{equation}
        \begin{aligned}
            \widehat{\mbf R}_T &= \sum_{t=1}^T \AngleBr{\theta_\star, x_\star - x_t} = \sum_{t=1}^{t_1} \AngleBr{\theta_\star,  x_\star - x_t} + \sum_{t=t_1+1}^T  \AngleBr{\theta_\star,  x_\star - x_t}  \\
            &\leq \RewardMax t_1 +  \sum_{t=t_1+1}^T  \AngleBr{\theta_\star - \widehat \theta_{t_1},  x_\star - x_t} + \sum_{t=t_1+1}^T  \underbrace{\AngleBr{\widehat \theta_{t_1},  x_\star - x_t}}_{\leq 0} \\
            &\leq \RewardMax t_1 + \sum_{t=t_1+1}^T  \AngleBr{\theta_\star - \widehat \theta_{t_1},  x_\star - x_t} \\
            &= \RewardMax t_1 + \sum_{t=t_1+1}^T  \AngleBr{\theta_\star - \widehat \theta_{t_1}, \Pi_{\olsi m}( x_\star - x_t)} \;
            &&\Big[\text{As $\theta_\star - \widehat \theta_{t_1} \in \olsi m$},  \Big] \\
            &\leq \RewardMax t_1 + \sum_{t=t_1+1}^T  \Norm{\theta_\star - \widehat \theta_{t_1}} \Norm{\Pi_{\olsi m}( x_\star - x_t)}  \\
            &\leq \RewardMax t_1 + \sum_{t=t_1+1}^T  2\sqrt{ \DiamX } \Norm{\theta_\star - \widehat \theta_{t_1}}; \;
            &&\Big[\text{Since $ \Norm{\Pi_{\olsi m}( x_\star - x_t)} \leq 2\sqrt{\DiamX} $ } \Big].
        \end{aligned}
    \end{equation}

    Now, we invoke Lemma \ref{Lem: Exploration Error}. Let $\mcal E$ is the event in that the exploration error is bounded as in Lemma \ref{Lem: Exploration Error}, then there is an absolute constant $c_1>0$ such that, 
    \begin{equation}
        \begin{aligned}
            \mbf R_T &\leq \RewardMax t_1 + \mbb E \SquareBr{ \sum_{t=t_1+1}^T  2\sqrt{\DiamX}\Norm{\theta_\star - \widehat \theta_{t_1}}  \bigg| \mcal E} + T \mrm{Pr}(\mcal E) \RewardMax \\
            &\leq \RewardMax t_1 + c_1 \sqrt{ \DiamX}   \sqrt{\frac{\sigma^2 \DimensionTrueFSS \log(d/\delta)}{C_{\min}(\mcal X)t_1}} T + T\delta \RewardMax.
        \end{aligned}
    \end{equation}

    Let $\delta = 1/T$, and $t_1 = \RewardMax^{-\frac{2}{3}} \sigma^{\frac{2}{3}}  C_{\min}^{-\frac{1}{3}}(\mcal X) \DiamX^{\frac{1}{3}} \DimensionTrueFSS^{\frac{1}{3}}  T^{\frac{2}{3}} (\log(dT))^{\frac{1}{3}}$, one has that
    \begin{equation}
        \mbf R_T = O\RoundBr{\RewardMax^{\frac{1}{3}}  \sigma^{\frac{2}{3}}  C_{\min}^{-\frac{1}{3}}(\mcal X) \DiamX^{\frac{1}{3}} \DimensionTrueFSS^{\frac{1}{3}}  T^{\frac{2}{3}}(\log(dT))^{\frac{1}{3}}}
    \end{equation}
\end{proof}

\section{Improved Regret Bound}
\subsection{Improve Regret Bound with Well-Separated Partitions}
\label{appendix: Improved Regret Bound}

\medskip

\ThmRegretUpperBoundWellDistinguishedPartition*

\begin{proof}
    
Define the event $\mcal E$ as 
\begin{equation}
    \mcal E  = \CurlyBr{\Norm{Y-\Pi_{S_{\widehat m}}(Y)}^2 \leq \Norm{Y-\Pi_{S_{m}}(Y)}^2 \mid \theta_\star \in \widehat m , \; \theta_\star \notin m},
\end{equation}
The event $\mcal E$ is equivalent as
\begin{equation} \label{eq: condition for alg return true subspace}
    \begin{aligned}
    &        &\Norm{Y-\Pi_{S_{\widehat m}}(Y)}^2 &\leq \Norm{Y-\Pi_{S_{m}}(Y)}^2 \\
    &\iff    &\Norm{\boldsymbol{f}_\star + \boldsymbol\eta - \boldsymbol{\widehat f}_{\widehat m}}^2 &\leq \Norm{\boldsymbol{f}_\star + \boldsymbol\eta - \boldsymbol{\widehat f}_{ m}}^2 \\
    &\iff    &\Norm{\boldsymbol{f}_\star - \boldsymbol{\widehat f}_{\widehat m}}^2 + 2\AngleBr{\boldsymbol\eta, \boldsymbol{f}_\star - \boldsymbol{\widehat f}_{\widehat m}} &\leq \Norm{\boldsymbol{f}_\star - \boldsymbol{\widehat f}_{ m}}^2 + 2\AngleBr{\boldsymbol\eta, \boldsymbol{f}_\star - \boldsymbol{\widehat f}_{ m}}.\\
    \end{aligned}
\end{equation}

\textbf{First}, we upper the LHS of \eqref{eq: condition for alg return true subspace} with high probability.
\begin{equation}
    \begin{aligned}
        \Norm{\boldsymbol{f}_\star - \boldsymbol{f}_\star - \Pi_{S_{\widehat m}}(\boldsymbol\eta)}^2 + 2\AngleBr{\boldsymbol\eta, - \Pi_{S_{\widehat m}}(\boldsymbol\eta)} &=  \Norm{\Pi_{S_{\widehat m}}(\boldsymbol\eta)}^2 - 2\Norm{\Pi_{S_{\widehat m}}(\boldsymbol\eta)}^2  \leq 0 \quad \text{[as $\theta_\star \in \widehat m$]}
    \end{aligned}
\end{equation}
with probability at least $1-\delta$. The above inequality uses the union bound for all $m \in \mcal M$.

\textbf{Second}, we lower bound the RHS of \eqref{eq: condition for alg return true subspace}, $\Norm{\boldsymbol{f}_\star  - \boldsymbol{\widehat f}_m}^2 + 2\AngleBr{\boldsymbol\eta, \boldsymbol{f}_\star - \boldsymbol{\widehat f}_m}$,  with high probability.
Let $\olsi S_m = S_m \bigoplus \AngleBr{\boldsymbol{f}_\star}$, then $\AngleBr{\boldsymbol\eta, \boldsymbol{f}_\star - \boldsymbol{\widehat f}_m} = \AngleBr{\Pi_{\olsi S_m}(\boldsymbol\eta), \boldsymbol{f}_\star - \boldsymbol{\widehat f}_m} $. Therefore, we have that, with probability at least $1-\delta/3$.
\begin{equation}
    \begin{aligned}
        \Norm{\boldsymbol{f}_\star  - \boldsymbol{\widehat f}_m}^2 + 2\AngleBr{\Pi_{\olsi S_m}(\boldsymbol\eta), \boldsymbol{f}_\star - \boldsymbol{\widehat f}_m} 
        &\geq \Norm{\boldsymbol{f}_\star  - \boldsymbol{\widehat f}_m}^2 - 2\Norm{\Pi_{\olsi S_m}(\boldsymbol\eta)}^2 - \frac{1}{2} \Norm{\boldsymbol{f}_\star  - \boldsymbol{\widehat f}_m}^2 \\
        &\geq \frac{1}{2} \Norm{\boldsymbol{f}_\star  - \boldsymbol{\widehat f}_m}^2 - 4\RoundBr{\sigma^2(\DimensionTrueFSS+1) + \sigma^2(\log(3M\delta^{-1}))} \\
        &\geq \frac{1}{2} \Norm{\boldsymbol{f}_\star  - \boldsymbol{\widehat f}_m}^2 - 5\RoundBr{\sigma^2 \DimensionTrueFSS + \sigma^2(\log(3M\delta^{-1}))}, \\
    \end{aligned}
\end{equation}
\text{as $d_m \leq \DimensionTrueFSS$}. 

Also, we have that, with probability at least $1-\delta/3$
\begin{equation*}
\begin{aligned}
    \Norm{\boldsymbol{f}_\star  - \boldsymbol{\widehat f}_m}^2 &= \Norm{\boldsymbol{f}_\star - \Pi_{ S_m}(\boldsymbol{f}_\star) - \Pi_{ S_m}(\boldsymbol\eta)}^2 \\ 
    &\geq  \Norm{\boldsymbol{f}_\star  - \Pi_{S_m}(\boldsymbol{f}_\star)}^2 \\
\end{aligned}
\end{equation*}
Where the inequalities holds because $\Pi_{S_m}(\boldsymbol{f}_\star)$ is the projection of $\boldsymbol{f}_\star$ to $S_m$.
Therefore, we have that
\begin{equation}
    \Norm{\boldsymbol{f}_\star  - \boldsymbol{\widehat f}_m}^2 + 2\AngleBr{\boldsymbol\eta, \boldsymbol{f}_\star - \boldsymbol{\widehat f}_m} \geq 
    \frac{1}{2}\Norm{\boldsymbol{f}_\star - \Pi_{ S_m}(\boldsymbol{f}_\star)}^2 - 5 \RoundBr{\sigma^2 \DimensionTrueFSS + \sigma^2(\log(M\delta^{-1}))} .
\end{equation}
Now, by choosing $t_2 = \Omega\RoundBr{ C_{\min}^{-2}(\mcal X) \DiamX^2\log(M\delta^{-1})}$, by Proposition \ref{prop: Concentration of RIP condition with arbitrary distribution}, we have that with probability at least $1-\delta/3$,
\begin{equation}
    \begin{aligned}
        \Norm{\boldsymbol{f}_\star - \Pi_{ S_m}(\boldsymbol{f}_\star)}^2 = \Norm{X(\theta_\star-\Pi_{S_m}(\theta_\star ))}^2 \geq \frac{t_2}{2C_{\min}(\mcal X)}\Norm{\theta_\star - \Pi_{S_m}(\theta_\star )}^2 \geq \frac{t_2\varepsilon_0^2}{4C_{\min}(\mcal X)}
    \end{aligned}
\end{equation}
Therefore, the RHS of \eqref{eq: condition for alg return true subspace} is lower bounded as follows
\begin{equation}
    \Norm{\boldsymbol{f}_\star  - \boldsymbol{\widehat f}_m}^2 + 2\AngleBr{\boldsymbol\eta, \boldsymbol{f}_\star - \boldsymbol{\widehat f}_m} \geq \frac{t_2\varepsilon_0^2}{8 C_{\min}(\mcal X)} - 5 \RoundBr{\sigma^2 \DimensionTrueFSS + \sigma^2(\log(3M\delta^{-1}))} .
\end{equation}

Therefore, the sufficient condition for event $\mcal E$ holds with probability at least $1-\delta$ is that
\begin{equation}
    \begin{aligned}
        \frac{t_2\varepsilon_0^2}{8 C_{\min}(\mcal X)} - 5 \RoundBr{\sigma^2 \DimensionTrueFSS + \sigma^2(\log(M\delta^{-1}))}  \geq 0;
    \end{aligned}
\end{equation}
Note that as $M = O(\DimensionTrueFSS\log(d))$, for $\mcal E$ holds with probability at least $1-\delta$, it suffice to choose 
\begin{equation}
    t_2 = \Omega\RoundBr{\frac{\sigma^2 \DiamX^2 \DimensionTrueFSS\log(d\delta^{-1})}{C_{\min}^2(\mcal X)\varepsilon_0^2}}.
\end{equation}

The regret upper bound is an immediate consequence of the fact that Algorithm \ref{alg:estc true model} return the true subspace $\hat m \ni \theta_\star$, combined with the regret bound of OFUL in the exploitation phase and $\delta = 1/T$.
\end{proof}

\subsection{Adapting to Separating Constant $\varepsilon_0$}
\label{appendix: Adapt with separating constant}

As stated in Theorem \ref{Thm: Regret Upper Bound of LTMC}, if one knows in advance that the separating constant $\varepsilon_0 \geq T^{-1/4}$, then using Algorithm \ref{alg:estc true model} leads to $\sqrt{T}$ regret. The reason is that the learner can learn the true subspace $m_\star$ after $\sqrt{T}$ steps with no mis-specification errors. However, without knowing $\varepsilon_0 \geq T^{-1/4}$ a priori, one cannot guarantee to recover the true subspace. Hence, naively using Algorithm \ref{alg:estc true model}, which is not aware of potential mis-specification errors, leads to linear regret.
On the other hand, using Algorithm \ref{alg: Explore Model Selection then Commit} can achieve regret $T^{2/3}$ in the worst case (without the knowledge of $\varepsilon_0 \geq T^{-1/4}$). A question arises: \textit{does there exist an algorithm that, without the knowledge of $\varepsilon_0$, can achieve regret $\sqrt{T}$ whenever $\varepsilon_0 \geq T^{-1/4}$, but guarantee the worst-case regret as $T^{2/3}$?}

We note that the role of $\varepsilon_0$ is similar to the minimum signal in sparsity, and it is somewhat surprising that the question of adapting to unknown minimum signal has not been resolved in the literature of sparse linear bandits. Towards answering the question, we propose a simple method using adaptation to misspecified error in linear bandit \cite{Foster2020_AdaptWithMisspecifiedErrorLinBandit}, which has a $\sqrt{T}$ regret whenever the separating constant is large, and enjoys a worst-case regret guarantee of slightly worse $T^{3/4}$ regret.

The algorithm described in \ref{alg: adaptive estc true model} is a direct application of the algorithm proposed in \cite{Foster2020_AdaptWithMisspecifiedErrorLinBandit}, designed for adapting to misspecification errors in linear bandits. 
Particularly, the algorithm in \cite{Foster2020_AdaptWithMisspecifiedErrorLinBandit} can adapt to unknown misspecification errors and achieve a regret bound of $\tilde{O}(\DimensionTrueFSS\sqrt{T} + \epsilon_{\mrm{mis}} T)$, where $\epsilon_{\mrm{mis}}$ is the misspecification error.
At a high level, our algorithm exploring $\sqrt{T}$ rounds using exploratory distribution, which ensures that the misspecification error $\epsilon_{\mrm{mis}}$ of the chosen subspace $\hat{m}$ is at most $T^{-1/4}$.
Therefore, we can run multiple linear bandit algorithms using different levels of misspecification error. 
Particularly, we use a collection of $K = \lfloor \log(T) \rfloor$ base algorithms, where a base algorithm $k \in [K]$ is a linear bandit algorithm with misspecified level $\varepsilon_k = 2^{-k}$. 
Note that the base algorithm $K$ has the same order of regret $\sqrt{T}$ as a well-specified model. 
Therefore, in exploitation phase, one can guarantee that in the case of well-separated partitions where $\epsilon_{\text{mis}} = 0$, the algorithm can achieve a regret of $\DimensionTrueFSS\sqrt{T}$, while in the general case, the regret caused by misspecification error is at most $T^{3/4}\sqrt{\DimensionTrueFSS}$.

{\small
\begin{algorithm}[h!]
\caption{Adaptive algorithm}
\begin{algorithmic}[1]\label{alg: adaptive estc true model}    
    \STATE Input $T,\;\nu,\; t_3$.
    \FOR{$t= 1, \cdots, t_3$}
    \STATE Independently pull arm $x_t$ according to $\nu$ and receive a reward $y_t$.
    \ENDFOR
    \STATE $X\leftarrow [x_1,..., x_{t_1}]^\top$, $Y\leftarrow [y_t]_{t\in [t_1]}$. 
    \STATE Compute $\widehat m$ as \eqref{eq: model selction procedure}.
    \STATE Let $K = \floor{\log(T^{1/4})}$, $\mcal E = \CurlyBr{\varepsilon_k:= 2^{-k}, \; k\in [K]}$. \\
    \FOR{$t=t_2+1$ to $T$}
	\STATE Corralling $K$ base misspecified linear bandit algorithms \texttt{SquareCB.Lin+($\varepsilon_k$)} \cite{Foster2020_AdaptWithMisspecifiedErrorLinBandit} on $\widehat m$.
    \ENDFOR
\end{algorithmic}
\end{algorithm}
}

\medskip

\begin{corollary} \label{Corollary: Regret Upper Bound of LTMC}
    Suppose the Assumptions \ref{assp: sub-exponential models}, \ref{assp: Cube-like bounded set of arms} hold. Then, there exists an algorithm which achieves regret bound as follows:
    \begin{itemize}
        \item[(i)][\textbf{Well-separated partitions}] If $\varepsilon_0 \geq T^{-1/4}$, then
            $\mbf R_T = \Tilde O( \DimensionTrueFSS\sqrt{T})$.
        \item[(ii)][\textbf{Non-well-separated partitions}] If $\varepsilon_0 < T^{-1/4}$, then
            $\mbf R_T = \Tilde O( \DimensionTrueFSS\sqrt{T} + T^{\frac{3}{4}} \sqrt{\DimensionTrueFSS})$.
    \end{itemize}
\end{corollary}
\begin{proof}[Proof sketch]
    Denote $\epsilon_{\mrm{mis}} = \|\theta_\star - \Pi_{\widehat m}(\theta_\star) \|_2$ as the misspecification error. 
    With 
    \[t_3 = \Omega\RoundBr{\sigma^2 \DimensionTrueFSS\log(d\delta^{-1})\sqrt{T}}, \] 
    We can guarantee that, with probability at least $1-\delta$, if:
    \begin{itemize}
        \item[(i)] If $\varepsilon_0 > T^{-1/4}$, then by Theorem \ref{Thm: Regret Upper Bound of LTMC}, $\theta_\star \in \widehat m$, that is, $\epsilon_{\mrm{mis}} = 0$;
        \item[(ii)] If $\varepsilon_0 > T^{-1/4}$, then by Lemma \ref{Lem: Exploration Error}, we can bound $\epsilon_{\mrm{mis}}  \leq T^{-1/4}$.
    \end{itemize}

The regret of adaptive algorithm in \citep{Foster2020_AdaptWithMisspecifiedErrorLinBandit} is of the form $\tilde{O}(\DimensionTrueFSS\sqrt{T} + \epsilon_{\mrm{mis}} T\sqrt{\DimensionTrueFSS})$.
Let $\delta = 1/T$.
Consider case (i) where $\epsilon_{\mrm{mis}} = 0$, we have $\mbf R_T = \Tilde O(\DimensionTrueFSS\sqrt{T})$. 
Consider case (ii), where $\epsilon_{\mrm{mis}} = T^{-1/4}$,  we have $\mbf R_T = \Tilde O(\DimensionTrueFSS\sqrt{T} + T^{3/4}\sqrt{\DimensionTrueFSS})$.
\end{proof}

The result in Corollary \ref{Corollary: Regret Upper Bound of LTMC} is still sub-optimal in the worst case, as it can only achieve $O(T^{3/4})$ regret bound instead of $O(T^{2/3})$. 
We conjecture that new techniques are required to achieve order-optimal regret in both cases, and will continue to investigate this question in future works.


\section{On Collections of Partitions with Subexponential-Size}
\label{appendix: Discussion}

\subsection{Important Classes of Partitions with Subexponential-Size}
\label{appendix: Structured Partition and Computation }

In this section, we discuss several important classes of partitions which satisfy Assumption \ref{assp: sub-exponential models}.

Pattern-avoidance partitions is arguable the most important class of studied partition \cite{Mansour2012_SetPartitionBook}, in which, non-crossing partition is one the most studied.
\medskip

\begin{definition}[\textbf{Non-Crossing Partition}]
   Let $[d]$ admits a cylic order as $1 < 2<...< d$, and $d<1$. A non-crossing partition of $[d]$ is a partition such that for if $i,\; j$ in one block and $p,\;q$ in one block, then they are not arranged in the order $i<p<j<q$.  
\end{definition}
Similarly, we denote $\NCPartitionSet_{d}$, $\NCPartitionSet_{d,k}$, $\NCPartitionSet_{d,\leq k}$  as the set of all non-crossing partition of $[d]$, the set of all partition of $[d]$ with $k$ classes, and  the set of all partition of $[d]$ with at most $k$ classes.
We have the following fact (\cite{Mansour2012_SetPartitionBook}' section 3.2).
\begin{equation}
    |\NCPartitionSet_d| = \frac{1}{d+1} {2d \choose d}, \quad |\NCPartitionSet_{d,k}| = \frac{1}{d} {d \choose k} {d \choose k-1},
\end{equation}
which is the Catalan number and the Narayana number.
Note that, 
\begin{equation}
    \frac{1}{d} {d \choose k} {d \choose k-1} \leq \frac{1}{d}\RoundBr{\frac{e d}{k}}^{k}\RoundBr{\frac{e d}{k-1}}^{k-1} \leq \frac{1}{d}\RoundBr{\frac{e d}{k}}^{2k} .
\end{equation}
Therefore, non-crossing partition satisfied the cardinality restriction as Assumption \ref{assp: sub-exponential models}, that is, \\$\NCPartitionSet_{d,\leq\DimensionTrueFSS} \subset \SubsetPartitionSet_{d,\leq\DimensionTrueFSS}$.

Another important class of pattern-avoidence partitions is nonnesting partition \cite{Chen2006_NonCrossingNestingPartition}.
\medskip

\begin{definition}[\textbf{Non-Nesting Partition}]
   Let $[d]$ admits a cylic order as $1 < 2<...< d$, and $d<1$. A non-crossing partition of $[d]$ is a partition such that for if $i,\; j$ in one block and $p,\;q$ in one block, then they are not arranged in the order $i<p<q<j$.  
\end{definition}
Similarly, we denote $\NNPartitionSet_{d}$, $\NNPartitionSet_{d,k}$, $\NNPartitionSet_{d,\leq k}$  as the set of all non-crossing partition of $[d]$, the set of all partition of $[d]$ with $k$ classes, and  the set of all partition of $[d]$ with at most $k$ classes.
There is bijections between class of $\NNPartitionSet_{d}$ and $\NCPartitionSet_{d}$, non-nesting partitions also satisfy the sub-exponential constraint (Assumption \ref{assp: sub-exponential models}).

One special case of both non-crossing partitions and non-nesting partitions is interval partition, which has \textit{identical structure as sparsity}.
\medskip

\begin{definition}[\textbf{Interval Partition}]
    A set partition of $[d]$ is an interval partition or partition of interval if its parts are interval.
\end{definition}
We denote $\IntervalPartition_{d}$ as the collection of all interval partition of $d$, we have that $\IntervalPartition_{d} \subset \NCPartitionSet_{d} \subset \PartitionSet_d$, and $\IntervalPartition_{d} \subset \NNPartitionSet_{d} \subset \PartitionSet_d$.
\medskip

\begin{remark}\label{remark: sparsity and interval partition}
$\IntervalPartition_d$ admits a Boolean lattice of order $2^{d-1}$, making it equivalent to the sparsity structure in $d-1$ dimensions. Specifically, consider the set of entries of parameters $\varphi \in \mathbb{R}^{d}$ with a linear order, that is, $\varphi_1 < \varphi_2 < \dots < \varphi_d$.
Then define the variable $\theta \in \mathbb{R}^{d-1}$ such that $\theta_i = (\varphi_{i+1} - \varphi_i)$. 
Each interval partition on the entries of $\varphi$ will determine a unique sparse pattern of $\theta$.
In other words, symmetric linear bandit is strictly harder than sparse bandit and inherits all the computational complexity challenges of sparse linear bandit, including the \textit{NP-hardness} of computational complexity.
\end{remark}


Inspired by the literature on sparse linear regression, where one can relax solving exact sparse linear regression by using norm-$1$ minimization, also known as LASSO methods, we ask whether there is a convex relaxation for the case of non-crossing partitions or pattern-avoidance partitions in general.

\subsection{Practical Examples of Partitions with Subexponential-Size}
\label{appendix: Practical Examples of Partitions with Subexponential-Size}
\paragraph{General hidden symmetries.} 
Examples of hidden symmetry in reinforcement learning tasks can be found in robotic control \cite{Mahajan2017_UnknownSymmetryRobotic, Abreu2023_UnknownSymmetryRobotic}, where robot is initially designed symmetrical, but part of symmetry is destroyed by mechanical imperfection.
Further examples of hidden symmetry can be also found in the literature on multi-agent reinforcement learning with a large number of agents.
To avoid the curse of dimensionality, researchers often rely on the assumption of the existence of homogeneous agents \cite{Chen2021_MeanFieldGamesAllHomogeneous,Mondal2022_MeanFieldGamesClassHomogeneous}.
In the extreme case where all agents are homogeneous, such as in mean-field games, sample complexity becomes independent  of the number of agents \cite{Chen2021_MeanFieldGamesAllHomogeneous}.
However, in practice, agents can be clustered into different types \cite{Mondal2022_MeanFieldGamesClassHomogeneous}, and this information may not be known in advance to the learner (here symmetry occurs between different agents from the same type).

\paragraph{Non-crossing partitions.}
Sub-exponential size naturally appears when there is a hierarchical structure on the set $[d]$, and the partitioning needs to respect this hierarchical structure.
Particularly, let $T(d,d_0)$ be the set of ordered trees with $(d+1)$ nodes and $d_0$ internal nodes (i.e., nodes that are not the leaves). 
A partition that respects an ordered tree groups the children of the same node into a single equivalence class (for example, see Figure \ref{fig: partition with ordered tree}).
It is shown in \cite{Dershowitz1986_PartitionWithTree} that the cardinality of the set of partitions that respect ordered trees in $T(d,d_0)$ is sub-exponential. More precisely, it is $O(d^{2d_0})$.
Furthermore, there is a natural bijection between partitions that respect ordered trees in $T(d,d_0)$ and the set of non-crossing partitions $\mcal{NC}_{d,d_0}$ \cite{Dershowitz1986_PartitionWithTree}.

Recall the subcontractor example in the introduction. 
Here, after the company hires subcontractors $\{1, 4, 6\}$ to do the job, these subcontractors further break down the tasks into smaller subtasks and hire additional subcontractors $\{2, 3\}$, $\{5\}$, and $\{7, 8, 9\}$, respectively, to execute the subtasks.

\paragraph{Non-nesting partitions.} Besides non-crossing partitions, another sub-exponential-size class of partitions with practical relevance is non-nesting partitions. Consider the resource allocation task where there are $d$ upcoming tasks and $d_0$ machines. The job of the designer is to allocate these tasks to each machine.

Now, assume each task will appear in time $t_1 < t_2 < \dots < t_d$, but the exact time (the value of $t_i$) is unknown to the designer. Moreover, the cost of machine $k \in [d_0]$, given a subset of tasks $A_k$ (ordered according to execution time), is $c_k = t_{\max(A_k)} - t_{\min(A_k)}$.

The goal of the designer is to minimize the maximum cost of all machines:
\[
\min \max_{k\in [d_0]} c_k.
\]
To achieve this, the designer should avoid nesting allocations (i.e., searching among non-nesting partitions). In particular, assuming that if tasks at times $t_i$ and $t_j$ are assigned to machine $k$, (where $t_i < t_j$), and tasks at times $t_p$ and $t_q$ are assigned to machine $k'$ (where $t_p < t_q$), then it should not be the case that $t_i < t_p < t_q < t_j$. 
This is because the cost of machine $k$ would be significantly higher than that of machine $k'$, and the cost could be reduced by swapping task $t_q$ for machine $k$ with task $t_j$ for machine $k'$.

\subsection{Efficient Greedy Algorithm for Specific Classes of Partitions}
\label{appendix: Efficient Greedy Algorithm}
The model selection procedure in the exploration phase of Algorithms \ref{alg: Explore Model Selection then Commit} and \ref{alg:estc true model} requires finding the best subspace in the pool $m \in \mathcal{M}$ with respect to least square errors. 
In the worst case, the algorithm needs to solve $M$ linear regression approaches. 
While the exact computation of the best subspace $\hat{m}$ as in \eqref{eq: model selction procedure} is an NP-hard problem in general (since it contains interval partition as a subclass), we argue that an greedy algorithm can find the ground-truth subspace $m_\star$ in $O(nd^5)$ time complexity, given sufficient large number of samples.

The pseudo code of the greedy algorithm is given in Figure \ref{alg: greedy search on lattice}. 
Given that the set of partitions $\SubsetPartitionSet_d$ is equipped with a lattice structure, in which the finest partition is $(1|2|\dots|d)$ and the coarsest partition is $(1,2,\dots,d)$. 

The algorithm starts with the finest partition $\hat{\pi} = (1|2|\dots|d)$. In each iteration, the algorithm finds the finest coarsening of the current partition $\hat{\pi}$. In graph-theoretic terms, it finds all neighbors of $\hat{\pi}$ in the lattice that are coarsenings of $\hat{\pi}$. The operator for finding the finest coarsening of $\hat{\pi}$ is denoted as \texttt{Coarsen}($\hat{\pi}$), and it returns a collection of the finest coarsened partitions.
Next, the algorithm finds the partition that minimizes the prediction error $\|Y - \Pi_{S_m}(Y)\|_2^2$ among the current coarsening collection. At the end of each iteration, as the number of classes in $\hat{\pi}$ is reduced by one, the dimension variable is also reduced by $1$.
The while loop stops when the dimension equals $d_0$.

Since the algorithm only optimizes locally within the current coarsening collection, it exhibits a behavior similar to a greedy algorithm. 
We note that for non-crossing and non-nesting partitions, the cardinality of the coarsening collection at any level of the lattice is at most $d^2$.
Therefore, assuming that creating a finest coarsening partition take $O(d)$ operator, and solving least square takes $O(nd^2)$, the algorithms time complexity is $O(n d^5)$.


{\small
\begin{algorithm}[h!]
\caption{Greedily Search within Lattice}
\begin{algorithmic}[1]\label{alg: adaptive estc true model}    
    \STATE Input: Compact representation of $\mathcal M$, design matrix $X$, reward vector $Y$.
    \STATE Initialise $\hat \pi = (1|2|3|4...|d)$, $\mrm{dimension} = d$. 
    \WHILE{$\mrm{dimension} > d_0$}
        \STATE Collection $=$ \texttt{Coarsen}($\hat \pi$).
        \STATE $\hat m = \argmin_{m \in \BijectionParFixSS(\mrm{Collection})} \|Y - \Pi_{S_m}(Y)\|_2^2$.
        \STATE $\hat \pi = \BijectionParFixSS^{-1}(\hat m)$.
        \STATE $\mrm{dimension} = \mrm{dimension} - 1$.
    \ENDWHILE
    \STATE $\hat \theta  = \argmin_{\theta \in \hat m} \|Y - X\theta\|_2^2$.
    \STATE Return $\hat m,\; \hat \theta$.
\end{algorithmic}
\label{alg: greedy search on lattice}
\end{algorithm}
}


\section{Experiment details}
\label{appendix: experiment}
We conduct simulations where the entries of $\theta_\star$ satisfy non-crossing partition constraints. The set of arms $\mcal X$ is $\sqrt{d}\mbb S^{d-1}$, $\sigma = 0.1$, and $(d,d_0)\in \{(40,4), (80,10), (100,15)\}$. 
We let exploratory distribution $\nu$ be the uniform distribution on the unit sphere. 
The ground-truth partition $\PartitionByGroup{\mcal G}$ and $\theta_\star$ are randomized before each simulation.

To run ESTC-Lasso algorithm \cite{Hao2020_SparseLinBandit_PoorRegime}, we introduce an auxiliary sparse vector $\varphi$ corresponding to $\theta_\star$, whose entries are defined as $\varphi_i = \theta_{i+1} - \theta_i$, and $\varphi_d = \theta_d$. 
We apply Lasso regression for $\varphi_\star$, get the estimate $\hat \varphi$, then convert back to $\hat \theta$ using the map that transforms sparse vector to interval-partition vector (inversion of the map we defined above). 

Regarding implementing our algorithm, we use greedy Algorithm \ref{alg: greedy search on lattice}  as introduced in Appendix \ref{appendix: Efficient Greedy Algorithm}, to solve the optimisation in equation \eqref{eq: f_m, theta_m}, \eqref{eq: model selction procedure}, and its complexity is $O(t_1 d^5)$. 
It is shown in the simulation result (Figure \ref{fig: simulation result for Noncrossing partition}, \ref{fig: simulation result d40}, \ref{fig: simulation result d80}), the greedy algorithm achieves small risk error and consequently leads to small regret.
Code is available at: \\ \url{https://github.com/NamTranKekL/Symmetric-Linear-Bandit-with-Hidden-Symmetry.git}.

\begin{figure}[h!] 
    \centering
    \includegraphics[width=1\linewidth]{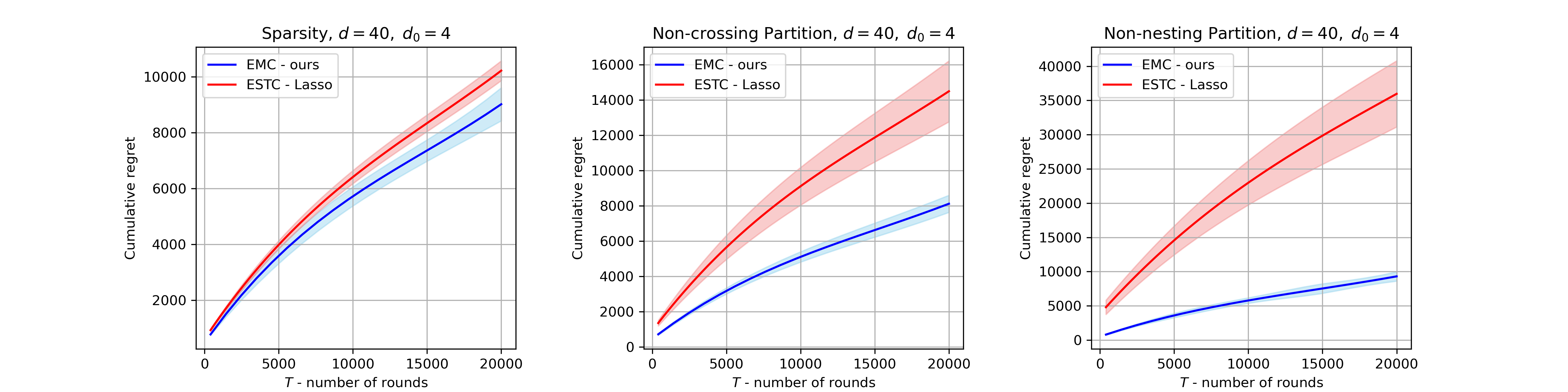}
    \caption{Regret of EMC (Algorithm \ref{alg: Explore Model Selection then Commit}) and of ESTC proposed in \cite{Hao2020_SparseLinBandit_PoorRegime}, in cases of sparsity, non-crossing partitions, and non-nesting partitions, with $d = 40,\; d_0 = 4$.}
    \label{fig: simulation result d40}
\end{figure}

\begin{figure}[h!] 
    \centering
    \includegraphics[width=1\linewidth]{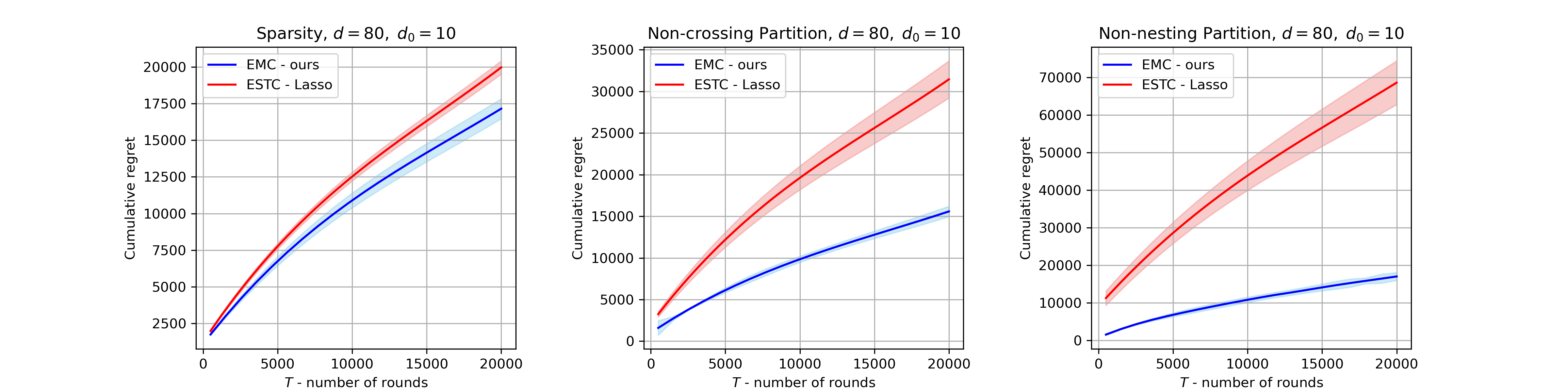}
    \caption{Regret of EMC (Algorithm \ref{alg: Explore Model Selection then Commit}) and of ESTC proposed in \cite{Hao2020_SparseLinBandit_PoorRegime}, in cases of sparsity, non-crossing partitions, and non-nesting partitions, with $d = 80,\; d_0 = 10$.}
    \label{fig: simulation result d80}
\end{figure}

\section{Extended Related Work}
\label{appendix: Extended Related Work}
\textbf{Sparse linear bandits.} As we will explain in Section \ref{sec: Regret Analysis of EMC algorithm}, sparsity is equivalent to a subset symmetry structures, and thus, can be seen as a special case of our setting. As such, we first review the literature of sparsity.
Sparse linear bandits were first investigated in \cite{Yasin_OnlineToConfidentSetSparseBandit}, where the authors achieve a regret of $\tilde{O}(\sqrt{dsT})$, with $\tilde{O}$ disregarding the logarithmic factor, and $s$ representing the sparsity level, and $T$ is the time horizon. 
Without additional assumptions on the arm set and feature distribution, the lower bound for regret in the sparsity case is $\Omega(\sqrt{dsT})$ \cite{Lattimore2020_BanditBook}. 
Consequently, the contextual setting has recently gained popularity in the sparsity literature, where additional assumptions are made regarding the context distribution and set of arms. 
With this assumption, it can be shown that one can achieve regret of the form $\tilde{O}(\tau s\sqrt{T})$, where $\tau$ is a problem-dependent constant that may have a complex form and varies from paper to paper \cite{Kim2019_SparseLinBanditcompatibility, Oh2020_AgnositcSparseLinearBandit}.
Apart from the contextual assumption, to avoid polynomial dependence on $T$ in the regret bound, assumptions are required for the set of arms \cite{Lattimore2015_SparseLinBanditCube, Carpentier2012_SparseLinBanditonSphere, Hao2020_SparseLinBandit_PoorRegime}.
Recently, \cite{Hao2020_SparseLinBandit_PoorRegime} offers a unified perspective on the assumption regarding the set of arms by assuming the existence of an exploratory distribution on the set of arms. With this assumption, the authors propose an Explore then Commit style strategy that achieves $\Tilde O(s^{\frac{2}{3}}T^{\frac{2}{3}})$, nearly matching the lower bound $\Omega(s^{\frac{2}{3}}T^{\frac{2}{3}})$ in the poor data regime \cite{Kyoungseok2022_PopArtSparseBandit}.
As the sparsity structure can be reduced to a subset of the symmetry structure, all the lower bounds for sparse problems apply to (unknown) symmetric problems.

\textbf{Model selection.} Our problem is also closely related to model selection in linear bandits, as the learner can collect potential candidates for the unknown symmetry.
Bandit model selection involves the problem where there is a collection of $M$ base algorithms (with unknown performance guarantees) and a master algorithm, aiming to perform as well as the best base algorithm.
The majority of the literature assumes the black-box collection of models $M$ base algorithms and employs a variant of online mirror descent to select the recommendations of the base agent \cite{Agarwal2016_BanditModelSelectionBlackBox, Pacchiano2020_BanditModelSelectionBlackBox, Pacchiano2023_BanditModelSelectionBlackBox}.
Due to the black-box nature, the regret guarantee bound depends on $\text{poly}(M)$.
There is a growing literature on model selection in stochastic linear bandits, where there is a collection of $M$ features, and linear bandits running with these features serve as base algorithms.
By exploiting the fact that the data can be shared across all the base algorithms, the dependence of regret in terms of the number of models can be reduced to $\log(M)$.
In particular, \cite{Kassraie2023_FeatureSelection_LinBan_LogM_Sparsity} propose a method that concatenates all $M$ features of dimension $d$ into one feature of dimension $Md$, and then runs a group-Lasso bandit algorithm on top of this concatenated feature space, using the Lasso estimation as a aggregation of models. 
Their algorithm achieves a regret bound of $O(T^{\frac{3}{4}} \sqrt{\log(M)})$ under the assumption that the Euclidean norm of the concatenated feature is bounded by a constant.
However, in our case, the Euclidean norm of concatenated feature can be as large as $\sqrt{M}$, which leads to $\sqrt{M}$ multiplicative factor in regret.
Besides, \cite{Moradipari2021_LinearBanditRepresentationLearning} uses the online aggregation oracle approach, and able to obtain regret as $O(\sqrt{KdT\log(M)})$, where $K$ is the number of arms.
In contrast, we use different algorithmic mechanism than aggregation of models. In particular, we explicitly exploiting the structure of the model class as a collection of subspaces
and invoking results from Gaussian model selection \citep{Giraud2021_HighDimStatBook} and dimension reduction on the union of subspaces \citep{Blumensath2009_RIP_Concentration_UoS}.
With this technique, we are able to achieve $O(T^{\frac{2}{3}}\log(M))$, which is rate-optimal in the data-poor regime, has logarithmic dependence on $M$ without strong assumptions on the norm of concatenated features, and is independent of the number of arms $K$.
A special case of feature selection where one can achieve a very tight regret compared to the best model is the nested feature class \cite{Foster2019_ModelSelectionLinBanditNestedClass, Ghosh2020_ModelSelectionLinBanditNestedClass}.
In particular, in the nested feature class where dimensions range from $\{1,\dots, d\}$, and $d_{m_\star} < d$ represents the realizable feature of the smallest dimension, the regret bound can be $\tilde{O}(\sqrt{T d_{m_\star}})$ as shown in \cite{Ghosh2020_ModelSelectionLinBanditNestedClass}.
While the regret bound nearly matches the regret of the best model in the nested feature class, the assumption on nested features cannot be applied in our setting.

\textbf{Symmetry in online learning.}  The notion of symmetry in Markov Decision Making dates back to works such as \cite{Givan2003_MDPBisimulation, Ravindran2004_MDPHomomorphism}. 
Generally, the reward function and probability transition are preserved under an action of a group on the state-action space. 
Exploiting known symmetry has been shown to help achieve better performance empirically \cite{Elise20_MDPHHomomorphism, Elise2020_MDPHomomorphismPlaning} or tighter regret bounds theoretically \cite{Tran2022_ILB}.
However, all these works requires knowledge of symmetry group, while setting consider unknown subgroup which may be considerably harder.
Unknown symmetry on the context or state space has been studied by few authors, with the term context-lumpable bandit \cite{Lee2023_Context_lumpable_bandit}, meaning that the set of contexts can be partitioned into classes of similar contexts.
It is important to note that the symmetry group acts differently on the context space and the action space.
As we shall explain in detail in Section \ref{sec: Partition and FixedSS, imposibility}, while one can achieve a reduction in terms of regret in the case of unknown symmetry acting on context spaces \cite{Lee2023_Context_lumpable_bandit}, this is not the case when the symmetry group acts on the action space.
Particularly, we show that without any extra information on the partition structure of similar classes of arms, no algorithm can achieve any reduction in terms of regret.

The work closest to ours is \cite{Pesquerel2021_MABsimilarArmsConstraintPartition}, where the authors consider the setting of a $K$-armed bandit, where the set of arms can be partitioned into groups with similar mean rewards, such that each group has at least $q > 2$ arms. With the constrained partition, the instance-dependent regret bounds are shown asymptotically to be of order $O\left((\nicefrac{K}{q}) \log T\right)$.
Comparing to \cite{Pesquerel2021_MABsimilarArmsConstraintPartition}, we study the setting of stochastic linear bandits with similar arms, in which the (unknown) symmetry and linearity structure may intertwine, making the problem more sophisticated. We also impose different constraints on the way one partitions the set of arms, which is more natural in the setting of linear bandits with infinite arms.
As a result, we argue that the technique used in \cite{Pesquerel2021_MABsimilarArmsConstraintPartition} cannot be applied in our setting.
To clarify this, let us first review the algorithmic technique from \cite{Pesquerel2021_MABsimilarArmsConstraintPartition}:
The algorithm in \cite{Pesquerel2021_MABsimilarArmsConstraintPartition} assumes there is an equivalence among the parameters $\theta$, and that the set of arms $\mcal{X}$ is a simplex. At each round $t$, given an estimation $\hat \theta_t$, the algorithm maintains a sorted list of indices in $[d]$ that follows the ascending order of the magnitude of $\hat \theta_i$.
The algorithm then uses the sorted list $(\hat \theta_i)_{i\in [d]}$ to choose arm $x$ accordingly. 
The key assumption here is that, since the set of arms $\mcal{X}$ is a simplex, we can estimate each $\theta_i$ independently. This implies that the list $(\hat\theta_i)_{i \in [d]}$ should respect the true order of the list $(\theta_i)_{i \in [d]}$ when there are a sufficiently large number of samples.
Unfortunately, this is typically not the case in linear bandits where $\mcal{X}$ has a more general shape. 
In linear bandits, there can be correlations between the estimates $\hat \theta_i$ and $\hat \theta_j$ for any $i, j \in [d]$. Hence, one should not expect that $(\hat \theta_i)_{i \in [d]}$ will maintain the same order as $(\theta_i)_{i \in [d]}$. In other words, the correlations among the estimates $\{\hat \theta_1, \ldots, \hat \theta_d\}$ may destroy the original order in $\{\theta_1, \ldots, \theta_d\}$.
In fact, we can only guarantee the risk error of estimation $\theta$, i.e., $\|\hat \theta - \theta_\star \|$ is small, but not necessarily the order of the indices in $\theta$.
Therefore, the technique used in \cite{Pesquerel2021_MABsimilarArmsConstraintPartition} cannot be directly applied to linear bandits in its current form.



\end{document}